\PassOptionsToPackage{capitalize,noabbrev}{cleveref}
\PassOptionsToPackage{cleveref}{jmlrutils}
\documentclass{alt2027}
\usepackage[T1]{fontenc}
\usepackage{times}
\usepackage{mathtools}
\usepackage{booktabs}
\usepackage{microtype}
\usepackage{xurl}
\hypersetup{
  pdftitle={Feature Normalization under Finite Training Budgets},
  pdfauthor={Johan Larsson}
}

\counterwithin{theorem}{section}
\aliascntresetthe{lemma}
\aliascntresetthe{proposition}
\aliascntresetthe{corollary}
\DeclareMathOperator{\diag}{diag}

\DeclareMathOperator{\E}{E}
\newcommand{\Prob}{\Pr}

\title{Feature Normalization under Finite Training Budgets}
\altauthor{
  \Name{Johan Larsson} \Email{johan@jolars.co}\\ \addr Department of
  Mathematical Sciences, University of Copenhagen, Universitetsparken 5, 2100
  Copenhagen, Denmark}

\begin{document}
\maketitle
\thispagestyle{plain}

\begin{abstract}
  Standardizing a rare feature can help stochastic gradient descent learn its
  effect within a finite budget, but it also changes the noise cost of learning.
  We characterize the lowest rare-subgroup risks attainable by standardized and
  unscaled training when each chooses a nonnegative constant learning rate and a
  deterministic stopping time under the same observation cap. For two
  independent binary features, we derive both optimized risks as prevalence
  vanishes while the expected number of rare observations remains finite.
  Uniform control over rates and stopping times reduces each limiting oracle to
  a scalar minimization. With positive noise, the preferred scaling changes
  exactly once as we move signal to the rare feature while the limiting risk at
  initialization stays fixed and positive. The formulas also explain how noise
  and budget affect the comparison. Separate finite-prevalence bounds favor
  standardization on a rare-dominant signal family and unscaled training on a
  common-dominant family. The results concern final iterates with population
  preprocessing and fresh observations. Exact-moment searches assess the
  approximation at finite prevalence. In a controlled study, both orderings
  persist under observation reuse, estimated preprocessing, and intercept
  fitting when tuning choices are fixed before evaluation. Selecting a predictor
  after observing training can reverse an ordering, even when selection uses
  exact population risk.
\end{abstract}

\begin{keywords}
  Feature normalization, stochastic gradient descent, early stopping,
  rare-subgroup risk
\end{keywords}

\section{Introduction}

A rare feature gives a model few opportunities to learn its effect. Within a
finite training budget, stochastic gradient descent may learn the effects of
common features while making little progress on the rare one. Standardizing the
rare feature gives those few observations more influence on its coefficient
updates. This can help recover the rare effect, but it also makes the fitted
coefficient more sensitive to noise.

These effects can substantially change prediction where the rare feature is
present, despite contributing little to average error. Normalization therefore
raises a statistical question: how strongly should training respond to scarce
evidence? For instance, scaling may leave a solution unchanged but changes which
signals we learn before stopping. \Cref{fig:intro-coefficient-paths} illustrates
this difference on a fixed least-squares problem: the methods approach the same
solution along different paths and reach different coefficients when stopped
early.

Adjusting the learning rate and stopping time is a natural way to compensate for
different learning speeds. We ask how low each scaling can drive rare-subgroup
risk when both optimize these choices under the same observation cap. In our
two-feature model, neither scaling is uniformly preferable.
\Cref{fig:intro-signal-crossing} illustrates the comparison: standardization
gives lower risk when the rare signal dominates, whereas unscaled training gives
lower risk when the common signal dominates.

The constraint is that a single learning rate controls both coefficient updates.
Relative feature scaling fixes the ratio of their steps, which tuning cannot
change. To learn a dominant rare signal, unscaled training also introduces noise
in the common coefficient. To learn a dominant common signal, standardization
introduces noise in the rare coefficient. A smaller rate reduces noise but slows
learning, and the observation cap limits how far a longer run can compensate.
The preferred scaling therefore depends on which signal is worth learning and
the noise cost of learning it.

\begin{figure}[htb]
  \centering
  \includegraphics{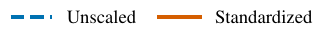}\par\smallskip
  \subfigure[Full-gradient descent on a noiseless dataset with binary features at prevalences $0.5$ and $0.02$. Dots mark an illustrative stop after 20 updates. Pale curves show progress toward the shared OLS solution,
    while gray contours show training loss.]{%
    \label{fig:intro-coefficient-paths}%
    \includegraphics{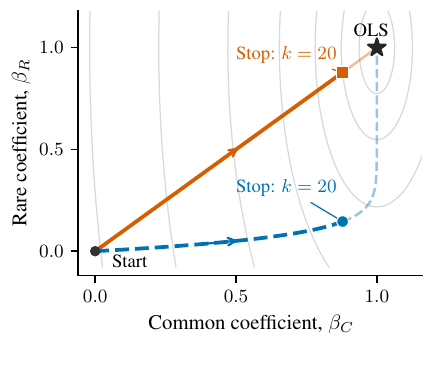}%
  }\hfill
  \subfigure[Limiting oracle risks for SGD with noise variance one and rescaled budget $\tau = 2$. As the rare signal share $\theta$ increases,
    $(\beta_C^*)^2 = 4(1 - \theta)$ and $(\beta_R^*)^2 = \theta$,
    fixing initial limiting subgroup risk at $1/2$. Lines connect saved scalar minima along the allocation path in \cref{cor:oracle-allocation}. \Cref{sec:signal-allocation} gives the setup and numerical checks.]{%
    \label{fig:intro-signal-crossing}%
    \includegraphics{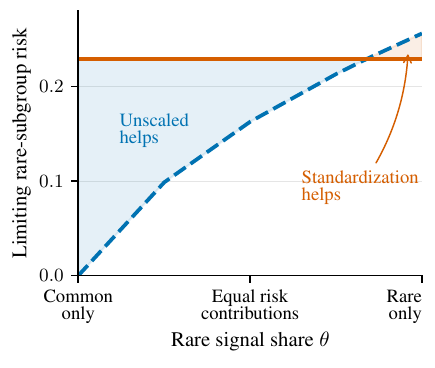}%
  }
  \caption{
    Normalization changes learning paths, and its effect on risk depends on the
    signal.
  }
\end{figure}

\begin{figure}[t]
  \centering
  \includegraphics{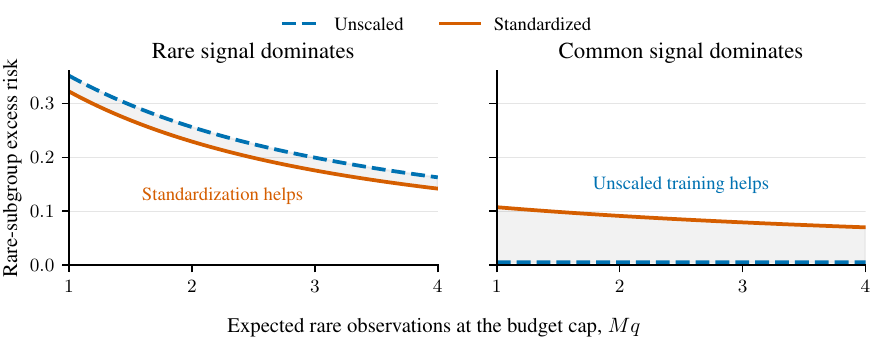}\par\smallskip
  \caption{
    Tuning does not erase the effect of normalization. Independent binary
    features have prevalences $1/2$ and $q = 10^{-4}$, with
    $(\beta_C^*,\beta_R^*) = (0.1,1)$ (left) or $(1,0.1)$ (right) and noise
    variance one. Under the stated population setup, curves show the lowest
    exact rare risk over a finite rate grid and every integer stop within each
    cap, including zero. Both panels use the same axes, and lower risk is
    better. This illustrates \cref{thm:conference-predictive}, which covers all
    nonnegative constant rates. See \cref{sec:predictive-intro-figure} for
    numerical details.
  }
  \label{fig:conference-tuned-intro}
\end{figure}

We study two independent binary predictors under squared loss, with known
population means and standard deviations, zero initialization, a fixed zero
intercept, and one fresh observation per update. Each method chooses a
nonnegative constant learning rate and a deterministic stopping time to minimize
the final iterate's rare-subgroup risk averaged over training randomness. We
call the resulting optimized risk its \emph{oracle risk}.

Our main result, \cref{thm:oracle-limits}, characterizes both oracle risks when
rare observations remain scarce even as the total observation cap grows. We let
prevalence $q$ decrease and increase the cap $M$ so that
$Mq(1 - q) \to \tau \in (0,\infty)$. The expected number of rare observations at
the cap then tends to $\tau$. For any fixed signal strengths and noise level,
each limiting oracle risk is given by minimizing an explicit function of one
variable. The characterization remains valid when the learning rate and stopping
time vary with prevalence.

We use these formulas to determine how the preferred scaling changes as signal
moves from the common feature to the rare feature. Along this path, we keep the
limiting rare-subgroup risk before training fixed and positive. With positive
noise and $\tau$ held fixed, the preferred scaling switches exactly once. The
formulas also show that, for fixed signal strengths, varying the noise level or
budget can change whether standardized or unscaled training has lower limiting
oracle risk.

The finite-prevalence comparison in \cref{thm:conference-predictive} complements
this characterization with explicit positive lower bounds on the difference
between the two oracle risks for two signal families. Standardization has lower
risk in the family with a dominant rare signal, and unscaled training has lower
risk in the family with a dominant common signal. The bounds hold whenever
$0 < q \leq 0.04$ and $1 \leq Mq(1 - q) \leq 4$.
\Cref{fig:conference-tuned-intro} illustrates these risk differences across
observation caps. As prevalence decreases, the budget grows while the expected
number of rare occurrences remains roughly between one and four. The risk
differences therefore persist even as the total observation cap grows.

We compute expected squared coefficient errors exactly to distinguish incomplete
learning from variability caused by feature sampling and observation noise.
These calculations also let us compare the limiting predictions with optimized
risks at finite prevalence. With several common features, we can still track
their combined error and the rare coefficient's error using two scalar
recurrences. Numerical searches through 64 common features continue to favor
standardization when the rare signal dominates and unscaled training when the
common signal dominates.

A controlled study then asks how these comparisons change when training reuses
observations, estimates preprocessing, or fits an intercept. We retain the two
illustrative signal pairs and find that both orderings persist when each
method's tuning choices are fixed before evaluation. We then allow selection
among the fitted predictors in each training run. This changes the objective:
the best predictor in a particular run can differ from the predictor trained
with the settings that work best on average. The two scalings can benefit by
different amounts, enough to reverse an ordering even when selection uses exact
population risk. We also examine how missing rare validation observations affect
selection.

\paragraph{Relation to prior work.}

\citet[Section 5.4]{bordes2010erratum} warn that standardizing very rare
features can hurt test performance. In a study of deterministic graph neural
network training, \citet[Section 5]{milligan2024normalization} identify
increased stochastic variance as a possible cost of standardizing rare binary
features. These observations motivate a comparison of the best risks attainable
under each scaling after tuning.

Scaling affects the regularization induced by finite training. Preconditioning
can change early-stopped predictions even when the least-squares optimum is
unchanged~\citep{wadia2021whitening}. Early stopping itself regularizes least
squares~\citep{ali2019early}, and \citet{ali2020stochastic} compare stochastic
gradient flow with ridge along the training path for a fixed design. For
explicit regularization, \citet{larsson2025normalization} analyze how feature
scaling affects shrinkage. Our question concerns scaling's effect when the rate
and stopping time are chosen under a finite observation cap.

Earlier comparisons after tuning show why the signal matters. In an
overparameterized gradient-flow model,
\citet[Proposition 7]{amari2020preconditioning} compare aggregate bias after
optimal early stopping for particular signal covariances.
\citet[Theorem 4.2]{su2024preconditioning} bound aggregate excess risk for
fresh-sample SGD with tail averaging and a restricted constant rate, then
compare preconditioned SGD with ridge regression. We compare the two scaling
rules using final iterates, optimizing over all nonnegative constant rates and
deterministic stops under a common cap.

Our focus on a rare subgroup connects the comparison to target-distribution risk
and covariate shift. We condition the test distribution on the rare feature's
presence while retaining the training distribution. \citet{lei2021shift} study
minimax linear estimation under target-weighted loss. \citet{wu2022pretraining}
and \citet[Theorem 4]{liu2025precondition} analyze SGD variants with decaying
rates under different training and test laws.

We use exact discrete-time SGD moments to separate incomplete learning,
feature-sampling variation, and observation noise. This builds on analyses of
least-mean-squares methods that separate bias and variance while accounting for
random sampling~\citep{defossez2015averaged, jain2018parallelizing}. As
prevalence vanishes, the rare feature's standardized fourth moment diverges, as
must the fourth-moment constant in \citet[Assumption 3.2]{su2024preconditioning}
in our model. Risk bounds whose constants depend on this moment need to retain
that dependence to describe our regime. Our exact moments retain it when passing
to optimized risks.

\section{Normalization, learning, and rare-subgroup risk}
\label{sec:conference-setting}

We isolate the effect of prevalence with two independent binary features: a
common feature $I$ with $\Prob(I = 1) = 1/2$ and a rare feature $J$ with
$\Prob(J = 1) = q$, where $0 < q < 1/2$. For the centered vector
$x = (I - 1/2,J - q)^\top$, suppose
\[
  y = x^\top\beta^* + \varepsilon,\qquad \E(\varepsilon \mid x)
    = 0,\qquad \E(\varepsilon^2 \mid x)
    = \sigma^2.
\]
Write $a = 1/4$, $b = q(1 - q)$, and $C = \E(xx^\top) = \diag(a,b)$. We use the
population means, fix the centered intercept at zero, and initialize
$\beta_0 = 0$.

\paragraph{Scaling and the training budget.}

We divide each feature by a positive scale factor, collected in the diagonal
matrix $S = \diag(s_C,s_R)$, and write $P = S^{-2}$. Training on $z = S^{-1}x$
gives the following SGD update on the original coefficient scale:
\begin{equation}
  \beta_{k+1} = \beta_k - \eta Px_k(x_k^\top\beta_k - y_k).
\end{equation}
We use squared loss with a factor $1/2$. Each update uses one independent fresh
observation, so an observation cap $M$ permits at most $M$ updates. We compare
unscaled features (U), with $P_U = I_2$, and standard-deviation scaling (S, or
SD), with $P_S = \diag(a^{-1},b^{-1})$. Both procedures use the final iterate.

We can compensate exactly for global rescaling by adjusting the learning rate.
Replacing $S$ by $gS$ and $\eta$ by $g^2\eta$ leaves every update unchanged when
the methods use the same data stream. Changing the relative feature scales,
however, generally changes $\eta P$ in a way that a single learning rate cannot
undo. Whether this leads to a risk difference after tuning requires a separate
comparison.

\paragraph{Prediction where the rare feature is present.}

Our main objective is excess prediction risk where the rare feature is present,
$J = 1$. For an independent test vector $x$, define
\[
  R_{R,k} = \frac12\E\bigl[ \bigl(x^\top(\beta_k - \beta^*)\bigr)^2 \mid J = 1 \bigr].
\]
The expectation averages over the conditional test distribution and training
randomness, and the risk excludes irreducible test noise. Independence keeps the
common feature centered within the rare subgroup, $\E(x_C \mid J = 1) = 0$, so
the cross term vanishes and
\begin{equation}
  R_{R,k} = \frac18\E\bigl[(\beta_{C,k} - \beta_C^*)^2\bigr]
            + \frac{(1 - q)^2}{2}\E\bigl[(\beta_{R,k} - \beta_R^*)^2\bigr].
\end{equation}
Here the expectations average over training randomness. Aggregate excess risk
$R_k$ replaces $(1 - q)^2$ by $q(1 - q)$ in this expression. Thus aggregate
error can give little weight to an error that matters within the rare subgroup.
We evaluate both objectives on the same trained predictor and retain the
original training distribution.

\paragraph{Learning signal and fitting noise.}

Let $e_k = \E(\beta_k - \beta^*)$ be the mean coefficient error. Under fresh
sampling, $e_{k+1,j} = (1 - \eta C_{jj}/s_j^2)e_{k,j}$. SD scaling therefore
gives every coordinate the multiplier $1 - \eta$. For $0 < \eta \leq 1$, early
stopping gives each coefficient the same shrinkage factor in expectation,
although its uncertainty still depends on feature sampling and observation
noise.

To separate these sources of uncertainty, let
$V_k = \E\bigl[(\beta_k - \beta^*)(\beta_k - \beta^*)^\top\bigr]$. The error
second moment decomposes as
\begin{equation}
  V_k = e_ke_k^\top + A_k + \sigma^2N_k, \label{eq:conference-decomposition}
\end{equation}
where $A_k \succeq 0$ is the covariance of the noiseless process under random
feature sampling and $\sigma^2N_k \succeq 0$ is the covariance contributed by
observation noise. This separates incomplete learning from the two sources of
training variation. \Cref{sec:variance-decomposition} derives the decomposition
and its moment formula.

To expose the tradeoff, temporarily fix $\eta_U = 4h$ and $\eta_S = h$, with
$0 < h \leq 1$. The common coefficient then has the same mean trajectory under
both rules. SD scaling changes the contraction factor for the mean error in the
rare coefficient from $1 - 4hb$ to $1 - h$, while increasing the
observation-noise variance added to that coefficient at each update from
$16h^2b\sigma^2$ to $h^2\sigma^2/b$. It learns the rare signal faster but also
introduces more noise.

\section{Oracle risks and the preferred scaling}

The matched-rate comparison leaves open how much tuning can compensate for these
differences. We therefore let each scaling choose its own constant learning rate
and stopping time under the same observation cap and compare the lowest
attainable rare-subgroup risks. We first establish opposite orderings for two
signal families at finite prevalence. We then characterize both optimized risks
as prevalence vanishes and use the formulas to locate a change in the preferred
scaling.

\paragraph{The oracle comparison.}

We parameterize each method's learning rate by its step on the rare coefficient,
$\alpha$, in the original coefficient scale. The update matrices are
\begin{equation}
  D_U(\alpha) =    \alpha I_2,\qquad D_S(\alpha)
              =    \diag(4b\alpha,\alpha),\quad \alpha
              \geq 0. \label{eq:conference-oracle-rates}
\end{equation}
Here $D = \eta P$, so $\eta_U = \alpha$ and $\eta_S = b\alpha$. Each
parameterization covers every nonnegative constant learning rate for its rule.
Holding $\alpha$ fixed gives both methods the same rare-coordinate mean
trajectory and the same observation-noise variance injected directly into that
coordinate at each update. SD's common-coordinate step is $4b$ times U's, so it
can reduce the common-coordinate observation noise introduced while learning the
rare effect.

Each oracle chooses $\alpha$ separately. For an observation cap
$M \in \mathbb N$, define
\begin{equation}
  R_A^*(M)
  =   \inf_{\substack{\alpha\geq0\\ k\in\{0,\ldots,M\}}} R_{R,k}^{A}(\alpha),\qquad A
  \in \{U,S\}.
\end{equation}
The oracle chooses a deterministic rate and stop using population risk averaged
over training randomness. It may stop at the prescribed zero initialization.
Because we evaluate finite stops, the search includes rates that need not be
mean-square stable. The cap $M$ applies to each trajectory; the infimum
describes attainable risk without charging for the search needed to find a rate
and stop. It does not compare methods with equal total computation for tuning.

\subsection{Quantitative gaps at finite prevalence}

To isolate the competing noise costs, we compare two signal families. In one,
the rare feature carries the dominant signal; in the other, the common feature
does. The dominant coefficient has magnitude equal to the noise standard
deviation, while the other has magnitude at most one tenth as large.

We restrict the finite-prevalence comparison to $1 \leq Mq(1 - q) \leq 4$,
corresponding to roughly one to four expected rare observations. The total cap
can therefore grow as prevalence decreases while evidence about the rare feature
remains scarce. On this interval, we certify uniform positive gaps for both
signal families. The endpoints specify sufficient bounds for these guarantees;
they need not mark a change in the preferred scaling or a threshold for learning
rare features.

The limiting characterization in \cref{thm:oracle-limits} covers every positive
rescaled budget. With positive noise and a nonzero rare signal, sufficiently
large budgets favor standardization in that limit even when the common signal
dominates (\cref{sec:oracle-noise-budget}). The following theorem establishes a
positive risk gap in each direction throughout the stated finite-prevalence
regime: standardization wins for the rare-dominant family, and unscaled training
wins for the common-dominant family, even after each method chooses its best
constant rate and deterministic stop.

\begin{theorem}[Opposite oracle rare-risk orderings]
  \label{thm:conference-predictive}
  Suppose $\sigma^2 > 0$, $0 < q \leq 0.04$, and
  $1 \leq \tau = Mq(1 - q) \leq 4$ in the independent binary model above.
  \begin{enumerate}
    \item If $(\beta_R^*)^2 = \sigma^2$ and $(\beta_C^*)^2 \leq \sigma^2/100$,
          then $R_U^*(M) - R_S^*(M) \geq \sigma^2/1000$.
    \item If $(\beta_C^*)^2 = \sigma^2$ and $(\beta_R^*)^2 \leq \sigma^2/100$,
          then $R_S^*(M) - R_U^*(M) \geq \sigma^2/50$.
  \end{enumerate}
\end{theorem}

The same bounds hold when each method restricts its search to mean-square-stable
nonnegative constant rates and initialization. The proof bounds the
disadvantaged method's risk below at every rate and stop and compares it with a
stable candidate for the other method.

The theorem identifies two regions with opposite orderings. To see the tradeoffs
within those regions, we examine the risk components at separately tuned
choices.

\paragraph{The compromise after tuning.}

\Cref{fig:tuned-decomposition} separates the risk components at the choices used
in \cref{fig:conference-tuned-intro} for one observation cap. With a dominant
rare signal, U incurs common-coordinate observation noise, while SD keeps that
contribution small and learns more of the rare signal. With a dominant common
signal, SD accepts substantial incomplete learning of that signal together with
rare-coordinate noise. U instead learns the common signal while leaving the
small rare effect largely unchanged. Thus the disadvantaged method need not have
noise as its largest risk component: tuning can favor leaving signal unlearned
to avoid the noise cost of further progress.

\begin{figure}[t]
  \centering
  \includegraphics{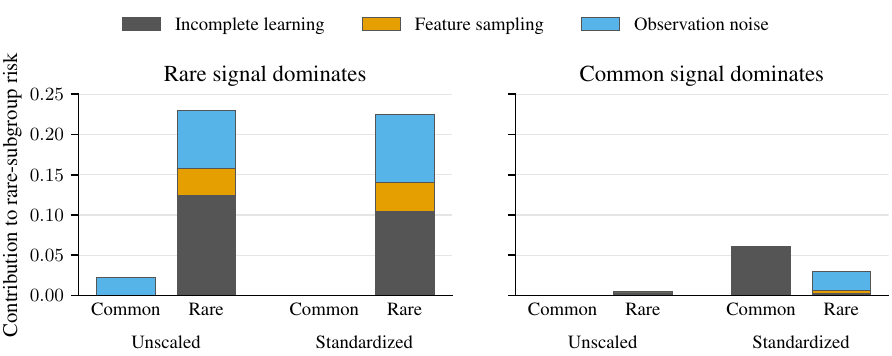}
  \caption{
    Normalization changes the balance between incomplete learning and training
    variability, even after tuning. Panels use $(\beta_C^*,\beta_R^*) = (0.1,1)$
    (left) and $(1,0.1)$ (right), with prevalence $q = 10^{-4}$, noise variance
    one, and observation cap $M = 20{,}503$. Each method selects its rate and
    stop to minimize exact rare-subgroup risk over the search grid used in
    \cref{fig:conference-tuned-intro}. Bars show the common and rare
    coefficients' contributions to that risk. Colors distinguish incomplete
    learning, feature-sampling variation, and observation noise, as in
    \cref{eq:conference-decomposition}. Each method's two bars sum to its total
    rare-subgroup risk.
  }
  \label{fig:tuned-decomposition}
\end{figure}

\paragraph{Establishing the gaps.}

An exact positive recurrence for the diagonal error moments makes this
comparison possible without finding the minimizing choices. Exact rational
interval calculations certify the constants over the continuous parameter
regions. \Cref{sec:predictive-proofs} gives the proof, certificates, and
stronger bounds.

\subsection{Limiting oracle risks}

The finite-prevalence theorem gives explicit gaps for two signal families. To
compare arbitrary fixed signals and noise, we let $q \to 0$ and choose caps
$M_q$ with $M_qb \to \tau \in (0,\infty)$. The total cap grows, but the expected
number of rare observations tends to $\tau$. This rescaled budget can now be any
positive value.

The rate parameterization makes the learning time scales explicit. For fixed
$0 < \alpha \leq 1$, U's common mean error contracts by $1 - \alpha/4$ per
update, while its rare mean error and both SD mean errors contract by
$1 - b\alpha$. U therefore learns the common signal on a shorter time scale than
the rare signal. Under SD, both signals are learned on the time scale set by
rare prevalence.

Write $s = (\beta_C^*)^2$, $r = (\beta_R^*)^2$, $\nu = \sigma^2$, and
$v = s/4 + r$. The limiting rare-subgroup risk at initialization is $v/2$. For
$\alpha \in[0,1]$, define
\[
  d(\alpha) = \alpha(2 - \alpha),\qquad z_\tau(\alpha)
            = e^{-\tau d(\alpha)},\qquad g(\alpha)
            = \frac{\nu\alpha}{8 - \alpha}.
\]
Here $z_\tau$ describes the decay of the initial error, and $g$ is the limiting
common-coordinate noise contribution under U, before the factor $1/2$ in the
risk. The limiting objectives are
\begin{equation}
  \begin{aligned}
    F_S(\alpha;s,r,\nu,\tau) & = \frac12\left[ vz_\tau + \frac{\nu\alpha}{2 - \alpha}(1 - z_\tau) \right],           \\
    F_U(\alpha;r,\nu,\tau)   & = \frac12\left[ g + rz_\tau + \frac{\alpha(\nu + g)}{2 - \alpha}(1 - z_\tau) \right].
  \end{aligned} \label{eq:limiting-objectives}
\end{equation}

\paragraph{Why the limits differ.}

The signal terms $vz_\tau$ and $rz_\tau$ include both squared mean error and
feature-sampling variation. SD's signal term depends on both initial errors
through $v$. U learns the common signal on its shorter time scale, so only $r$
remains in its signal term. It pays common-coordinate observation noise $g$,
which also feeds into the rare error through $\nu + g$. SD's observation-noise
cost is the second term in $F_S$.

For U, the endpoint $\alpha = 0$ represents a limit of small positive rates.
Taking $b \ll \alpha_q \ll 1$ at the full cap learns the common signal with
vanishing noise cost while leaving the rare effect unchanged. This gives
$F_U(0) = r/2$. Literal zero-rate training retains both signal errors and has
limiting risk $v/2$.

The next theorem shows that minimizing these objectives gives the limits of the
full oracle searches over rates and stopping times.

\begin{theorem}[Limiting oracle risks]
  \label{thm:oracle-limits}
  In the model of \cref{sec:conference-setting}, fix finite $s,r,\nu \geq 0$. If
  $q \to 0$ and integer observation caps satisfy $M_qb \to \tau \in (0,\infty)$,
  then
  \[
    R_A^*(M_q)
    \longrightarrow L_A(\tau)
    :=              \min_{0 \leq \alpha \leq 1}F_A(\alpha;\tau), \qquad A
    \in             \{U,S\},
  \]
  with the signal and noise arguments as in \cref{eq:limiting-objectives}. More
  precisely, for every $0 < T_0 < T_1 < \infty$, uniformly over integer caps
  with $T_q = M_qb \in[T_0,T_1]$,
  \begin{equation}
    |R_S^*(M_q) - L_S(T_q)| \leq C_Sb,\qquad |R_U^*(M_q) - L_U(T_q)|
                            \leq C_Ub\log(1/b) \label{eq:oracle-limit-errors}
  \end{equation}
  for all sufficiently small $q$. The constants depend only on
  $s,r,\nu,T_0,T_1$.
\end{theorem}

The theorem reduces the limiting optimization to $\alpha \in[0,1]$ and the full
cap. This reduction follows from bounds over the original search, which still
includes every nonnegative constant rate and deterministic stop, including
choices that depend on prevalence and unstable rates. The error bounds describe
convergence as prevalence falls. Their constants are not optimized to certify a
gap at a specified prevalence; \cref{thm:conference-predictive} provides that
separate guarantee.

For $\nu > 0$, U's limiting minimizer is zero exactly when $r\tau \leq \nu/16$;
above that threshold, its unique minimizer is positive. SD has a unique positive
minimizer whenever $v,\nu > 0$. The scalar equations in
\cref{sec:oracle-variational} determine these rates without a grid search.

\paragraph{Proof outline.}

For a fixed positive $\alpha$, let $G$ and $H$ denote the common and rare
contributions before the factor $1/2$ in the limiting risk. Derivatives are
taken in rescaled time $t = kb$. Under SD, their sum $W = G + H$ obeys
$\dot W = -d(\alpha)W + \alpha^2\nu$, starting at $W(0) = v$. Under U, the
common moment rapidly approaches $g(\alpha)$, and the rare moment obeys
$\dot H = -d(\alpha)H + \alpha^2(\nu + g)$, starting at $H(0) = r$. Solving
these equations gives \cref{eq:limiting-objectives}.

To pass from these fixed-rate limits to the oracle risks, we also need bounds
that hold uniformly over large rates, early stops, and vanishing rates. For U,
these bounds track the common moment as it approaches $g$. We then construct
rates that attain each limiting minimum, including its zero endpoint.
\Cref{sec:oracle-limit-proofs} gives the full argument.

\subsection{A unique signal-allocation crossing}

Hold the limiting subgroup risk at initialization fixed while moving signal
between the features. With $v > 0$, write $r = \theta v$ and
$s = 4(1 - \theta)v$ for $\theta \in[0,1]$. Thus $\theta$ measures the rare
feature's share of $s/4 + r$, rather than its share of the unweighted squared
coefficient sum.

\begin{corollary}[Signal allocation]
  \label{cor:oracle-allocation}
  Fix $v,\nu,\tau > 0$. Along this allocation path, $L_S$ is constant and $L_U$
  increases strictly with $\theta$. There is a unique $\theta_* \in (0,1)$ such
  that $L_S = L_U$. Unscaled training has lower limiting risk for
  $\theta < \theta_*$, and standardization has lower limiting risk for
  $\theta > \theta_*$.
\end{corollary}

\begin{proof}
  The SD objective depends only on $v$. Each U objective has slope
  $vz_\tau(\alpha)/2 \in[ve^{-\tau}/2,v/2]$ in $\theta$, so its minimum is
  continuous and strictly increasing. At $\theta = 0$, U has zero risk and SD
  has positive risk. At $\theta = 1$, a positive U minimizer incurs strictly
  more noise than the same SD rate. If U instead minimizes at zero, a small
  positive SD rate improves on $v/2$. The opposite signs at the endpoints give
  the unique crossing.
\end{proof}

\Cref{fig:intro-signal-crossing} illustrates this crossing. At $v = \nu = 1$ and
$\tau = 2$, the crossing is numerically $\theta_* \approx 0.8377$.
\Cref{sec:oracle-variational} gives a variational formula and an equivalent
boundary obtained by varying the common signal at fixed rare signal.

\paragraph{Noise and budget.}

Changing noise or budget can also change the preferred scaling. With a positive
common signal, U wins at sufficiently small and sufficiently large noise, while
SD can win between them. For fixed $r,\nu > 0$ and finite $s \geq 0$, a large
enough rescaled budget favors SD: $L_S(\tau)/L_U(\tau) \to 4/5$ as
$\tau \to \infty$, although convergence can be slow. This statement takes
$q \to 0$ before $\tau \to \infty$ and does not cover joint limits with
$M_qb \to \infty$. \Cref{sec:oracle-noise-budget} gives the expansions and
examples.

\paragraph{Finite prevalence and more common features.}

Exact-moment searches at six prevalences assess the limiting predictions. At the
smallest prevalence, they agree with the limiting choice of scaling for the
tested cases away from ties. At larger prevalences, a high-noise example favors
SD on the finite grids while the limit favors U. These searches assess the
approximation without certifying the continuous optima
(\cref{sec:oracle-finite}).

Replacing the common feature with $m$ independent balanced binary features
preserves an exact recurrence for their combined error and the rare error. With
the total squared common signal and the rare signal held fixed, increasing $m$
cannot lower either method's oracle risk. Numerical searches retain both
orderings through $m = 64$, although their gaps change with dimension.
\Cref{sec:common-block-details} gives the recurrence, proof, and numerical
comparisons; the two-feature separation bounds and limiting characterization are
not asserted for this larger model.

The next section examines how the two orderings respond to changes in training
and to selection among the realized predictors.

\section{From oracle tuning to validation selection}

The oracle fixes its rate and stop before observing training. Validation instead
selects among fitted predictors. We first ask whether the two risk orderings
survive changes to the training procedure when tuning choices are fixed in
advance. We then examine how selection among fitted predictors changes the
comparison, including when validation contains no rare observations.

\subsection{Controlled departures from the oracle model}
\label{sec:theory-bridge}

We change three aspects of training: whether observations are reused, whether
preprocessing is estimated, and whether an intercept is fitted. Eight training
arms introduce these changes separately and in combination. We retain the
independent binary model, coefficient pairs $(0.1,1)$ and $(1,0.1)$, and noise
variance one. At $q = 0.02$ and $0.005$, the caps give roughly one, two, or four
expected rare occurrences. Arms that reuse observations draw from 256 noisy
training rows. Empirical preprocessing uses all 256 rows, so its data
requirements can exceed the SGD observation cap.
\Cref{sec:theory-bridge-protocol} gives the full protocol.

For each scaling and training arm, we calibrate a rate factor and stop to
minimize mean rare risk, then evaluate those fixed choices on independent
training runs. With estimated preprocessing, the factor determines the constant
learning rate from the training data.

Both orderings persist across all eight training arms. SD has lower rare risk
for the dominant rare signal, and U has lower rare risk for the dominant common
signal (\cref{fig:bridge-controlled} in the appendix). This holds in all 96
comparisons across the arms, signal pairs, prevalences, and caps. These
numerical comparisons use a conservative finite rate search and cover only the
two tested signal pairs. Recalibrating on disjoint halves of the saved
calibration runs preserves the ordering in all splits for 83 of the 84
comparisons with estimated tuning choices. One small gap reverses in 2 of 200
splits (\cref{sec:calibration-sensitivity}).

Selecting among the fitted predictors can reverse an ordering even with fresh
observations and population preprocessing. \Cref{tab:bridge-selection} compares
the fixed-intercept baseline with the arm that adds only a fitted intercept,
initialized at zero. With the fitted intercept, calibration favors SD, while
both rare validation loss and exact population rare risk select predictors that
give U lower risk on average. Exact selection changes the SD-minus-U gap from
$-0.1046$ to $+0.0164$, reducing U's rare risk by $0.2399$ and SD's by $0.1189$.
U's larger gain overturns SD's advantage under calibration. In the baseline,
exact selection narrows the gap from $-0.0192$ to $-0.0059$, preserving SD's
advantage. Both orderings persist under exact selection in all 12 baseline
settings.

The reported gains can combine adaptation to each training run with correction
of a suboptimal calibration choice. For the fitted-intercept setting in
\cref{tab:bridge-selection}, minimizing exact expected risk over the same finite
rate-and-stop family gives $0.23184$ for U and $0.13823$ for SD\@. These are
expectations over training randomness, distinct from the table's Monte Carlo
means. SD therefore retains its advantage at the best fixed choices, so
calibration suboptimality alone cannot explain the reversal
(\cref{sec:bridge-exact-calibration}).

Every validation set in this example contains rare observations. The reversal
under exact selection shows that validation noise is not needed: choosing the
best predictor in each training run is a different objective from choosing the
best tuning settings on average. The comparison isolates intercept fitting as
sufficient for a reversal in this setting. U gains more than SD in $71.3\%$ of
runs under exact selection, with median gains $0.1456$ and $0.0626$,
respectively. The selected U predictors have smaller mean prediction errors
within the rare group, while SD retains a slightly smaller error in the common
coefficient. For U, the intercept and rare-coefficient errors nearly cancel: the
cross term removes $99.9\%$ of U's contributions to mean squared error from
those two coordinates, compared with $27.4\%$ for SD\@. This cancellation
accounts for the selected risk ordering even though U has larger errors in each
coordinate separately (\cref{sec:bridge-cancellation}). \Cref{sec:bridge-gains}
examines the gain distributions.

\begin{table}[t]
  \centering
  \begin{tabular}{lrrrrrr}
    \toprule
                           & \multicolumn{2}{c}{Rare risk} & \multicolumn{2}{c}{Reduction from calibration} & \multicolumn{2}{c}{SD-minus-U gap} \\
    \cmidrule(lr){2-3} \cmidrule(lr){4-5} \cmidrule(lr){6-7}
    Selection              &        U &       SD &        U &       SD &      Mean &      MCSE \\
    \midrule
    \multicolumn{7}{l}{\emph{Fresh observations, fitted intercept}}                            \\
    Separate calibration   & $0.2409$ & $0.1364$ &      --- &      --- & $-0.1046$ & $0.00364$ \\
    Validation (2048 rows) & $0.0268$ & $0.0344$ & $0.2141$ & $0.1020$ & $+0.0075$ & $0.00132$ \\
    Exact population risk  & $0.0011$ & $0.0175$ & $0.2399$ & $0.1189$ & $+0.0164$ & $0.00128$ \\
    \midrule
    \multicolumn{7}{l}{\emph{Fresh observations, fixed zero intercept (baseline)}}             \\
    Separate calibration   & $0.1528$ & $0.1337$ &      --- &      --- & $-0.0192$ & $0.00070$ \\
    Validation (2048 rows) & $0.0393$ & $0.0259$ & $0.1136$ & $0.1078$ & $-0.0134$ & $0.00065$ \\
    Exact population risk  & $0.0147$ & $0.0088$ & $0.1381$ & $0.1248$ & $-0.0059$ & $0.00056$ \\
    \bottomrule
  \end{tabular}
  \caption{
    Rare risk and gains from selection with $(\beta_C^*,\beta_R^*) = (0.1,1)$,
    $q = 0.02$, and $M = 204$, with fresh observations and population
    preprocessing. The fitted-intercept arm and baseline (arms 2 and 1 in
    \cref{tab:bridge-arms}) share the observation cap and differ only in whether
    an intercept is fitted; all coefficients are initialized at zero. Positive
    reductions indicate improvement over the separately calibrated choice for
    the same scaling; negative SD-minus-U gaps favor SD\@. Means use 4096 paired
    evaluation runs. MCSEs refer to the paired gaps and condition on the
    calibration choices. \Cref{sec:bridge-exact-calibration} gives exact
    expected-risk calibration benchmarks for the same candidate family.
  }
  \label{tab:bridge-selection}
\end{table}

\paragraph{Missing rare validation observations.}

Validation can also fail to represent the subgroup whose risk we want to
minimize. For the fresh baseline with a dominant rare signal, $q = 0.005$, and
$M = 804$, more than half of the validation sets with 128 rows contain no rare
observations: the group is absent in 2206 of 4096 sets. In those runs, selection
falls back to aggregate loss.

When the rare group is present, the SD-minus-U gap is $-0.105$, favoring SD.
When it is absent, the gap is $+2.193$, favoring U. Overall, validation
selection changes the gap from $-0.0216$ under calibration to $+1.133$ (MCSE
$0.062$) (\cref{sec:theory-bridge-protocol,fig:bridge-validation}).

\section{Discussion}

Normalization changes the noise cost of learning each signal. In the
rare-prevalence limit, unscaled training can learn the common signal with
vanishing noise cost while leaving the rare effect unchanged. Standardization
can instead avoid the common-coordinate noise cost of learning a rare effect.
Noise and budget determine which compromise is preferable, even after each
method has optimized its learning rate and stopping time.

The limiting formulas make this comparison explicit for arbitrary fixed signals
and noise. With positive noise, moving signal to the rare feature changes the
preferred scaling exactly once when the limiting risk at initialization is held
fixed and positive. The finite-prevalence theorem establishes positive gaps for
two signal families while rare observations remain scarce. Together, these
results show both why the preferred scaling changes and where one scaling has a
guaranteed advantage.

Both theorems concern two independent binary features, known population
preprocessing, zero initialization, a fixed zero intercept, and final iterates
under fresh sampling. Applying the limiting comparison at a finite prevalence
also requires care: numerical searches can favor a different scaling even when
the limiting risks are separated. The large-budget ratio $4/5$ takes prevalence
to zero before letting the rescaled budget grow. Extending the comparison to
joint limits, learning-rate schedules, averaging, batch selection,
data-dependent stopping, or adaptive optimizers requires further analysis.

Replacing the common feature with an independent block preserves the exact
moment reduction. At fixed total signal strength, increasing the number of
common features cannot lower either method's oracle risk. The numerical
comparison retains both orderings through 64 common features, with dimension
changing the size of the gaps. Uniform separation bounds for this larger model
remain open.

The controlled study shows that both orderings can survive changes to the
training procedure. They persist for the tested signal pairs under observation
reuse, estimated preprocessing, and intercept fitting when tuning choices are
fixed before evaluation. Selecting a predictor after observing training can
reverse an ordering, even with exact population risk. An extension to adaptive
selection would therefore need to account for the benefit each scaling gains
from choosing among its realized predictors.

Feature normalization therefore belongs in the statistical specification of
training under a finite budget: even in a two-feature model, tuning a global
learning rate and stopping time need not make different relative scalings
equivalent in rare-subgroup risk.

\section*{AI Disclosure}

We used generative AI tools, including OpenAI Codex, to plan research, search
the literature, derive mathematical results and review proofs, implement
experiments and independent numerical checks, and edit the manuscript.

\clearpage
\typeout{ALT-MAIN-PAGES: \number\numexpr\value{page}-1\relax}
\ifnum\value{page}>13
\PackageWarningNoLine{normreg-multi}{Draft main text exceeds the first 12 pages
  that ALT reviewers are expected to read}
\fi

\bibliography{references}

\clearpage

\appendix
\crefalias{section}{appendix}

\section*{Supplementary Materials}

The supplement gives the proof and certificates for the finite-prevalence oracle
comparison (\cref{sec:predictive-proofs}), the proof of the limiting oracle
characterization (\cref{sec:oracle-limit-proofs}), and its variational
consequences (\cref{sec:oracle-variational}). \Cref{sec:common-block-details}
gives the common-feature extension and its numerical study.
\Cref{sec:predictive-intro-figure} describes the numerical illustrations and
finite-prevalence checks of the limiting predictions.
\Cref{sec:empirical-methods} specifies the controlled study's training,
calibration, and selection protocols and reports its numerical checks.
Reproduction instructions are in \cref{sec:reproducibility}.

\section{Proof of the comparison after oracle tuning}
\label{sec:predictive-proofs}

We prove \cref{thm:conference-predictive} directly at positive prevalence. The
bounds below are explicit functions of $q$ and $\tau = Mq(1 - q)$. Finite minima
over rate intervals bound the continuous search; they do not restrict the oracle
to a rate grid. Analytic bounds cover every rate above one. An independent
checker certifies the reported positive regions with rational arithmetic and
checks the moment bounds against observation-level tensor calculations.

The source supplement also includes independent small-prevalence checks.

\subsection{Exact moments}

By homogeneity, take $\sigma^2 = 1$. Put $s = (\beta_C^*)^2$,
$r = (\beta_R^*)^2$, $b = q(1 - q)$, and $w = (1 - q)^2$. Write
$C_k = \E\bigl((\beta_{C,k} - \beta_C^*)^2\bigr)$ and
$H_k = \E\bigl((\beta_{R,k} - \beta_R^*)^2\bigr)$. Independence gives
\begin{equation}
  R_{R,k} = C_k/8 + wH_k/2.
\end{equation}
The required feature moments are $x_C^2 = 1/4$, $\E(x_R^2) = b$,
$\E(x_R^4) = b(1 - 3b)$, and $\E(x_C^2x_R^2) = b/4$. Mixed moments with an odd
power of $x_C$ vanish. Expanding the squared update and using the conditional
noise assumptions gives
\begin{equation}
  \binom{C_{k+1}}{H_{k+1}} = \begin{pmatrix}
                               A           & 4u^2b \\
                               b\alpha^2/4 & D_b
                             \end{pmatrix}\binom{C_k}{H_k}
                             + \binom{4u^2}{b\alpha^2},\qquad \binom{C_0}{H_0}
                           = \binom{s}{r}, \label{eq:predictive-moments}
\end{equation}
where $A = (1 - u)^2$, $D_b = 1 - b\alpha(2 - \alpha + 3b\alpha)$, and
$u = \alpha/4$ for U or $u = b\alpha$ for S. The diagonal equations and risk do
not involve cross moments of the coefficient errors, so coefficient signs do not
affect risk. Every matrix entry is nonnegative for every $\alpha \geq 0$: in
particular, $D_b = (1 - b\alpha)^2 + b\alpha^2(1 - 4b) \geq 0$. Increasing
either initial squared coefficient cannot reduce risk.

For $0 < \alpha \leq 1$, the row sums are
\[
  1 - u\{2 - u(1 + 4b)\} < 1,\qquad 1 - b\alpha\{2 - 5\alpha/4 + 3b\alpha\} < 1.
\]
The homogeneous moments contract and the second moments due to noise stay
bounded. The cross moment is bounded by $\sqrt{C_kH_k}$, so all positive
candidate rates used below are mean-square stable. At zero rate, the predictor
stays at initialization.

\subsection{Explicit finite-prevalence bounds}
\label{sec:predictive-finite-bounds}

Throughout, $0 < q \leq 1/16$ and $1 \leq \tau \leq 4$. Let $N = 4096$ and
$a_j = j/N$, for $j = 0,\ldots,N$. Define
\[
  \begin{aligned}
    d_b(a) & = a(2 - a + 3ba),                                                            \\
    n_j    & = \min\left\{ w,\frac{wa_j}{2 - a_j + 3ba_j} + \frac{a_j}{8 - a_j} \right\}, \\
    p_j    & = \min\left\{ \frac18,\frac{wa_j}{2(2 - ba_j)} \right\},
  \end{aligned}
\]
and
\[
  z_{R,j} = \exp\left\{ -\frac{\tau d_b(a_{j+1})}
                                  {1 - bd_b(a_{j+1})} \right\},\qquad z_{C,j}
          = \exp\left\{ -\frac{2\tau a_{j+1}}{1 - ba_{j+1}} \right\}.
\]
Define the lower bounds
\begin{equation}
  \begin{aligned}
    L_R(q,\tau) & = \min\left\{ \frac{w}{2(1 + 3b)}, \min_{0 \leq j<N}\frac{n_j + (w - n_j)z_{R,j}}2 \right\}, \\
    L_C(q,\tau) & = \min_{0 \leq j<N}\left\{ p_j + \left( \frac18 - p_j \right)z_{C,j} \right\}.
  \end{aligned} \label{eq:predictive-finite-lower}
\end{equation}
The integer $N$ controls only the precision of the lower bounds. All
denominators in these bounds are positive on the stated domain.

For the upper bounds, use $k = M$ and the stable rates
\[
  \alpha_S = \frac{2}{\tau + 4},\qquad \alpha_U = \frac{8b}{\tau} = \frac8M.
\]
Put
\[
  \begin{gathered}
    p_S = \frac{\alpha_S(1 - q)}{2 - (1 - q)\alpha_S},\qquad G = \frac{\alpha_U}{2 - \alpha_U/4}, \\
    H_* = \frac{1/100 + b\alpha_U/(2 - \alpha_U/4) + \tau\alpha_U^2(1 + G/4)}
               {1 - \tau b\alpha_U^2G/4}.
  \end{gathered}
\]
Here $\alpha_S \leq 2/5$, $\alpha_U \leq 1/2$, and
$\tau b\alpha_U^2G/4 \leq 1/960$, so $H_*$ is well defined. Define
\begin{equation}
  \begin{aligned}
    V_S(q,\tau) & = \frac12\left[ p_S + \left( w + \frac1{400} - p_S \right) e^{-\tau\alpha_S\{2-(1-q)\alpha_S\}} \right], \\
    V_U(q,\tau) & = \frac{e^{-4}}8 + \frac G8 + \left( \frac w2 + \frac{bG}8 \right)H_*.
  \end{aligned}
\end{equation}
Define the gap bounds
\begin{equation}
  \Delta_R(q,\tau) = L_R(q,\tau) - V_S(q,\tau),\qquad \Delta_C(q,\tau)
                   = L_C(q,\tau) - V_U(q,\tau).
\end{equation}
For the rare-signal and common-signal families, respectively, we prove
\begin{equation}
  \begin{aligned}
    R_U^*(M) - R_S^*(M) & \geq \Delta_R(q,\tau), \\
    R_S^*(M) - R_U^*(M) & \geq \Delta_C(q,\tau),
  \end{aligned} \label{eq:predictive-oracle-gaps}
\end{equation}
throughout $0 < q \leq 1/16$ and $1 \leq \tau \leq 4$, with $\sigma^2 = 1$ as
above. For general positive noise variance, multiply both bounds by $\sigma^2$.
The bounds need not be positive everywhere on this domain. A nonpositive bound
does not establish the opposite ordering.

\subsection{Dominant rare signal}

Suppose $r = 1$ and $s \leq 1/100$. For U and $\alpha \in[0,1]$,
\[
  D_b - A = \frac{(1 - 4b)\alpha}{16} \{8 + (12b - 1)\alpha\} \geq 0.
\]
Dropping nonnegative coupling terms in \cref{eq:predictive-moments} and summing
the direct noise terms gives, with $z = D_b^k$,
\begin{equation}
  R_{R,k}^U \geq \frac12\left[ wz + \left\{ \frac{w\alpha}{2 - \alpha + 3b\alpha} + \frac{\alpha}{8 - \alpha} \right\}(1 - z) \right].
\end{equation}
For the common-noise term, use $1 - A^k \geq 1 - D_b^k$. The formula extends
directly to zero rate. Its right side is affine in $z$. Taking the smaller of
its values at $z = 1$ and $z = D_b^M$ is equivalent to capping the noise
coefficient at $w$.

For $0 \leq x < 1$, the logarithm series gives $-\log(1 - x) \leq x/(1 - x)$.
Consequently,
\[
  D_b^M \geq \exp\left\{ -\frac{\tau d_b(\alpha)}
                                      {1 - bd_b(\alpha)} \right\}.
\]
Both $d_b(\alpha)$ and the noise coefficient increase on $[0,1]$. For
$\alpha \in[a_j,a_{j+1}]$, use the left endpoint for the capped noise
coefficient and the right endpoint for the exponential lower bound. This gives
the $j$th term in $L_R$ for every $0 \leq k \leq M$.

For every $\alpha \geq 1$, put $m_b = 1/(1 + 3b)$. The identity
\[
  D_bm_b + b\alpha^2 - m_b = \frac{2b\alpha(\alpha - 1)}{1 + 3b} \geq 0
\]
and positivity of the recurrence imply $H_k \geq m_b$ by induction from
$H_0 = 1$. Thus $R_{R,k}^U \geq w/\{2(1 + 3b)\}$ at every finite stop, including
zero. This proves the remaining term in $L_R$, without a stability restriction.

For the S candidate, retain the exact risk weight by setting
$W_k = C_k/4 + wH_k = 2R_{R,k}^S$. Let
$\rho = 1 - b\alpha\{2 - (1 - q)\alpha\}$. The recurrence gives
\[
  W_{k+1} \leq \rho W_k + b\alpha^2(1 - q),
\]
because the common coefficient has multiplier $\rho$ and
\[
  \rho w - (wD_b + b^3\alpha^2) = 2wbq(1 - 2q)\alpha^2 \geq 0.
\]
At $\alpha = \alpha_S$, $0 < \rho < 1$ and $W_0 \leq w + 1/400$. Geometric
summation therefore bounds the risk by
\[
  \frac12\left[ p_S + (w + 1/400 - p_S)\rho^M \right] \leq V_S(q,\tau).
\]
The last step uses $\rho^M \leq e^{-\tau\alpha_S\{2-(1-q)\alpha_S\}}$ and
$w + 1/400 > p_S$. Hence the risk at every U rate and stop exceeds this S
candidate's risk by at least $\Delta_R(q,\tau)$, uniformly over the rare-signal
family.

\subsection{Dominant common signal}

Suppose $s = 1$ and $r \leq 1/100$. Under S, put $A = (1 - b\alpha)^2$. Since
$D_b - A = b\alpha^2(1 - 4b) \geq 0$, positivity gives
\[
  C_k \geq A^k,\qquad H_k
      \geq b\alpha^2\sum_{j=0}^{k-1}D_b^j
      \geq b\alpha^2\sum_{j=0}^{k-1}A^j.
\]
For $0 < \alpha \leq 1$, geometric summation yields
\[
  R_{R,k}^S \geq \frac{A^k}{8} + \frac{w\alpha}{2(2 - b\alpha)}(1 - A^k).
\]
The expression is affine in $A^k$. Minimizing over $k \leq M$ caps its second
coefficient at $1/8$ and replaces $A^k$ by $A^M$. Moreover,
\[
  A^M \geq \exp\left\{ -\frac{2\tau\alpha}{1 - b\alpha} \right\}.
\]
The plateau and the magnitude of the exponent both increase with $\alpha$. Their
left and right endpoints give the $j$th term in $L_C$ on each rate interval.
Zero rate gives risk at least $1/8$ directly.

For $\alpha \geq 1$, if $A \geq 1$ then $C_k \geq 1$ and the risk is at least
$1/8$. Otherwise $0 < b\alpha < 2$ and the same geometric bound gives
\[
  R_{R,k}^S \geq \min\left\{ \frac18, \frac{w\alpha}{2(2 - b\alpha)} \right\}
            \geq \min\left\{ \frac18,\frac{w}{2(2 - b)} \right\}
            =    \frac18.
\]
The last equality follows from $q \leq 1/16$. Since $L_C \leq 1/8$, the bound
covers all larger rates as well.

For the U candidate, set $\alpha = \alpha_U$, $u = \alpha/4$, and
$H_{\max} = \max_{0 \leq j \leq M}H_j$. Summing the common equation gives
\[
  C_j \leq A^j + G(1 + bH_{\max}).
\]
Since $D_b \leq 1$, drop the rare-coordinate contraction and sum up to any
$j \leq M$. Using $\sum_{j \geq 0}A^j = 1/\{u(2 - u)\}$ gives
\[
  H_{\max} \leq \frac1{100}
                + \frac{b\alpha}{2 - \alpha/4}
                + \tau\alpha^2\left( 1 + \frac G4 \right)
                + \frac{\tau b\alpha^2G}{4}H_{\max}.
\]
Thus $H_{\max} \leq H_*$. Also,
$A^M = (1 - \alpha/4)^{2M} \leq e^{-\alpha M/2} = e^{-4}$, which gives
$R_{R,M}^U \leq V_U(q,\tau)$. The risk at every S rate and stop therefore
exceeds this U candidate's risk by at least $\Delta_C(q,\tau)$, uniformly over
the common-signal family.

Taking the two oracle infima proves \cref{eq:predictive-oracle-gaps}.
Restricting each search to stable rates preserves both lower bounds and both
candidates. Rescaling by $\sigma^2$ gives the bounds for general positive noise
variance. The certificates below then imply \cref{thm:conference-predictive}.
The argument requires neither stability of the losing method's rate nor
attainment of either infimum.

\subsection{Positive regions and the small-prevalence corollary}
\label{sec:predictive-certification}

\Cref{tab:predictive-certificates} summarizes the checker's certified lower
bounds for the two signal families.
\begin{table}[htbp]
  \centering
  \begin{tabular}{@{}lrrr@{}}
    \toprule
    Family & Maximum prevalence $Q$ &               Gap bound & Certificate leaves \\
    \midrule
    Rare   &                 $1/20$ &  $\Delta_R \geq 1/1000$ &         $77{,}510$ \\
    Common &                 $1/25$ &    $\Delta_C \geq 1/50$ &              $103$ \\
    Rare   &                 $1/50$ &  $\Delta_R \geq 7/1000$ &         $76{,}623$ \\
    Common &                 $1/50$ &   $\Delta_C \geq 9/200$ &               $58$ \\
    Rare   &              $10^{-4}$ & $\Delta_R \geq 11/1000$ &          $7{,}158$ \\
    Common &              $10^{-4}$ &  $\Delta_C \geq 11/200$ &              $105$ \\
    \bottomrule
  \end{tabular}
  \caption{
    Certified gap bounds for $0 < q \leq Q$, $1 \leq \tau \leq 4$, and
    $\sigma^2 = 1$. Each bound holds throughout the corresponding signal family.
    Certificate leaves count the accepted boxes in the interval certificates.
  }
  \label{tab:predictive-certificates}
\end{table}

Here is the full certification rule. On a box
$[a_0,a_1] \times[q_0,q_1] \times[t_0,t_1]$, set $b_i = q_i(1 - q_i)$ and
$w_i = (1 - q_i)^2$. For the rare lower bound, use $w_1$, the capped noise
coefficient at $(a_0,q_1)$, and the exponential lower bound at $(a_1,q_1,t_1)$.
For the common lower bound, use $\min\{1/8,w_1a_0/[2(2 - b_0a_0)]\}$ for the
plateau and the exponential lower bound at $(a_1,q_1,t_1)$. These substitutions
follow from monotonicity of $b$, $d_b$, the rate plateaus, and the exponents on
the stated ranges. Also include the rare large-rate floor at $q_1$.

The S upper bound decreases with $q$ and $\tau$: its plateau decreases, its
initial coefficient decreases, and its exponent
$4\tau(\tau + 3 + q)/(\tau + 4)^2$ increases. Use $V_S(q_0,t_0)$ on the box. For
the U upper bound, substitute $b_1$ for $b$, $w_0$ for $w$, and $t_0$ for
$\tau$. Indeed, $\alpha_U$, $G$, the numerator of $H_*$, and its feedback
coefficient increase with $b$ and decrease with $\tau$ after substituting
$\alpha_U = 8b/\tau$.

To enclose exponentials, for rational $x \in[0,4]$ put
$P_{18}(x) = \sum_{j=0}^{18}x^j/j!$. Positivity and the ratio of successive
omitted terms give
\begin{equation}
  P_{18}(x) \leq e^x \leq P_{18}(x) + \frac{x^{19}/19!}{1 - x/20}.
\end{equation}
Taking reciprocals encloses $e^{-x}$. The checker applies this rule to one
fourth of each exponent and raises the endpoints to the fourth power. All
exponents lie in $[0,16]$, so the rule covers every evaluation using rational
arithmetic only.

Start with $[0,1] \times[0,Q] \times[1,4]$ for each table row; the formulas
extend continuously to $q = 0$. Accept a box only if its lower margin reaches
the stated gap. Otherwise bisect the dimension with largest width among
$a_1 - a_0$, $q_1 - q_0$, and $(t_1 - t_0)/3$, breaking ties in that order. Stop
bisecting rate intervals at width $1/N$. All rate endpoints therefore align with
the partition defining \cref{eq:predictive-finite-lower}; an accepted box bounds
every constituent interval's expression. The accepted boxes cover the domain,
and their volumes sum exactly to $3Q$. Thus the certificates establish the
bounds over continuous prevalence and budget ranges, including every admissible
integer stop. No sampled rate or prevalence search is used as proof.

At $q = 1/50$ and $M = 102$, direct evaluation of the same formulas with
rational exponential enclosures gives $\Delta_R > 0.0092472$ and
$\Delta_C > 0.0641380$. These are certified family bounds, not estimates of the
oracle gaps. As a separate diagnostic, the checker compares 888 bounds with
80-digit observation-level tensor powers, including zero stops, both signal
families, and rates above the stability range.

\begin{corollary}[Small prevalence]
  Under the assumptions of \cref{thm:conference-predictive}, if
  $0 < q \leq 10^{-4}$, then the rare-signal gap is at least $11\sigma^2/1000$
  and the common-signal gap is at least $11\sigma^2/200$.
\end{corollary}

The last two rows of \cref{tab:predictive-certificates} prove the corollary. The
certified gaps grow as the prevalence range narrows. These are sufficient
bounds, not the largest possible gaps or prevalence ranges: the proof discards
nonnegative coupling terms and replaces discrete powers by exponential bounds.

\section{Proof of the limiting oracle characterization}
\label{sec:oracle-limit-proofs}

We prove \cref{thm:oracle-limits} uniformly over the original rate-and-stop
domain. Throughout, $s,r,\nu \geq 0$ are fixed and finite, $b = q(1 - q)$,
$w = (1 - q)^2$, $v = s/4 + r$, and $T = M_qb \in[T_0,T_1]$ for fixed
$0 < T_0 < T_1 < \infty$. Constants in the $O(\cdot)$ terms may depend on these
fixed quantities, but not on the rate, stop, or cap in the stated ranges.

\subsection{Exact moments and large-rate bounds}

Rescale the common moment in \cref{eq:predictive-moments} by setting
$G_k = C_k/4$ and restore the arbitrary noise variance $\nu$. The recurrence is
\begin{equation}
  \begin{aligned}
    G_{k+1}     & = (1 - u)^2G_k + u^2bH_k + u^2\nu,       \\
    H_{k+1}     & = D_b(\alpha)H_k + b\alpha^2(G_k + \nu), \\
    D_b(\alpha) & = 1 - b\alpha(2 - \alpha + 3b\alpha),
  \end{aligned}
\end{equation}
with $u = b\alpha$ for S and $u = \alpha/4$ for U. Initially
$(G_0,H_0) = (s/4,r)$, and $R_{R,k} = (G_k + wH_k)/2$. All coefficients are
nonnegative, since $D_b = (1 - b\alpha)^2 + b\alpha^2(1 - 4b) \geq 0$.

For S, let $W_k = G_k + H_k$. Summing the two equations gives
\[
  W_{k+1} \geq \lambda_b(\alpha)W_k + b\alpha^2\nu, \qquad \lambda_b(\alpha)
          =    1 - 2b\alpha + b\alpha^2(1 - 3b + b^2).
\]
The coefficient of $H_k$ is exactly $\lambda_b$, and that of $G_k$ exceeds it by
$b^2\alpha^2(4 - b)$. The exact noise term is $b\alpha^2(1 + b)\nu$. Also
$\lambda_b \geq 1 - b/(1 - 3b + b^2) > 0$ for $0 < b \leq 1/4$. Consequently,
\[
  R^S_{R,k}(\alpha) \geq \frac w2\left[ \lambda_b^kv + b\alpha^2\nu\sum_{j=0}^{k-1}\lambda_b^j \right].
\]
In particular, for $\alpha \geq c_S(b) := 2/(1 - 3b + b^2)$, this risk is at
least $wv/2$ at every stop, including initialization.

For U, positivity gives the analogous bound
\[
  R^U_{R,k}(\alpha) \geq \frac w2\left[ D_b^kr + b\alpha^2\nu\sum_{j=0}^{k-1}D_b^j \right].
\]
If $\alpha \geq c_U(b) := 2/(1 - 3b)$, then $D_b \geq 1$ and the risk is at
least $wr/2$. Empty sums cover $k = 0$. These bounds require no stability
assumption and hold even when rates diverge or stopping times are chosen to
offset growth in the moments. For sufficiently small $q$, both cutoffs are below
three and $1 - w \leq 4b$.

\subsection{Uniform approximation for standardization}

For $\alpha \in[0,3]$, the affine transition for $(G,H,1)$ under S is
$I + b\mathcal L_\alpha + O(b^2)$ uniformly, where
\[
  \mathcal L_\alpha = \begin{pmatrix}
                        -2\alpha & 0          & 0           \\
                        \alpha^2 & -d(\alpha) & \alpha^2\nu \\
                        0        & 0          & 0
                      \end{pmatrix}.
\]
Its difference from $e^{b\mathcal L_\alpha}$ is $O(b^2)$, and both operators
have norm at most $1 + O(b)$. Telescoping their powers over $k \leq T_1/b$ steps
therefore gives an $O(b)$ error uniformly in rate and stop. At time $t = kb$,
the limiting moments solve
\[
  \dot G = -2\alpha G,\qquad \dot H = -d(\alpha)H + \alpha^2(G + \nu).
\]
Their sum solves
\[
  \dot W = -d(\alpha)W + \alpha^2\nu,\qquad W(0) = v.
\]
Define $\Psi_S(\alpha,t) = W(t)/2$. The risk weights differ from their limits by
$O(b)$, so
\begin{equation}
  R^S_{R,k}(\alpha) = \Psi_S(\alpha,kb) + O(b) \label{eq:limit-s-uniform}
\end{equation}
uniformly over the retained rates and all allowed stops. For $0 < \alpha < 2$,
solving the scalar equation gives $F_S$ with budget $t$. The flow also defines
the endpoints smoothly: $\Psi_S(0,t) = v/2$ and $\Psi_S(2,t) = (v + 4\nu t)/2$.
A vanishing SD rate retains both initial moments uniformly over all stops.

\subsection{Uniform approximation with the unscaled transient}

For U, retain $0 \leq \alpha \leq c_U(b)$ and write
\[
  A = (1 - \alpha/4)^2,\qquad h
    = \frac{\alpha}{8 - \alpha},\qquad g
    = \nu h,\qquad c
    = bh, \qquad B
    = \frac{r + T\alpha^2(G_0 + g + \nu)}{1 - T\alpha^2c}.
\]
Here $0 \leq D_b \leq 1$, $h \leq 3/5$, and $T\alpha^2c \leq 27T_1b/5$. The
denominator is at least $1/2$ for small $b$. With
$H_{\max} = \max_{j \leq M_q}H_j$, summing the common recurrence gives
\[
  0 \leq G_j - A^jG_0 - (1 - A^j)g \leq cH_{\max}.
\]
Using $G_j \leq G_0 + g + cH_{\max}$ in the rare recurrence then gives
\[
  H_{\max} \leq r + T\alpha^2(G_0 + g + \nu + cH_{\max}).
\]
Absorbing the last term yields
\[
  H_{\max} \leq B
           \leq 2\{r + 9T_1(G_0 + 8\nu/5)\},\qquad 0
           \leq G_j - A^jG_0 - (1 - A^j)g
           \leq cB.
\]

Let $\widetilde H_0 = r$ and
$\widetilde H_{j+1} = D_b\widetilde H_j + b\alpha^2(\nu + g)$. Since
$D_b \leq 1$,
\begin{align*}
  |H_k - \widetilde H_k| & \leq b\alpha^2\sum_{j<k}|G_j - g|            \\
                         & \leq b\{16h|G_0 - g| + T\alpha^2hB\} = O(b).
\end{align*}
For $\alpha > 0$, the geometric sum uses $\alpha^2/(1 - A) = 16h$. At
$\alpha = 0$, both sides are zero. This cancellation removes the apparent
singularity in the common-coordinate relaxation time.

Let $H^\circ_\alpha$ solve
\[
  \dot H^\circ = -d(\alpha)H^\circ + \alpha^2(\nu + g), \qquad H^\circ(0) = r.
\]
The scalar affine map for $\widetilde H$ differs from this flow over time $b$ by
$O(b^2)$ uniformly on the retained rate interval. The flow is uniformly bounded
for $t \leq T_1$, and errors propagate with multiplier $D_b \leq 1$. Telescoping
gives $\widetilde H_k = H^\circ_\alpha(kb) + O(b)$. Hence
\begin{equation}
  R^U_{R,k}(\alpha) = \frac12\{A^kG_0 + (1 - A^k)g + H^\circ_\alpha(kb)\}
                      + O(b) \label{eq:limit-u-uniform}
\end{equation}
uniformly over every retained rate and stop. A uniform approximation that
replaces the common term by $g$ would fail, including at initialization.

This approximation also describes sequences approaching the rate or time
boundaries. If $\alpha_q \to 0$, positivity and Bernoulli's inequality give
\[
  r - 2rT\alpha_q \leq H_j
                  \leq B
                  =    r + O(\alpha_q^2),\qquad 0
                  \leq G_j - A^jG_0
                  \leq g + cB
                  =    O(\alpha_q).
\]
If $k_q\alpha_q \to \ell \in[0,\infty]$, the limiting risk is therefore
\begin{equation}
  R^U_{R,k_q}(\alpha_q) \longrightarrow \frac{s}{8}e^{-\ell/2} + \frac r2.
\end{equation}
For finite $\ell$, use $\log A = -\alpha_q/2 + O(\alpha_q^2)$ and
$k_q\alpha_q^2 \to 0$. For infinite $\ell$, use
$A^{k_q} \leq e^{-k_q\alpha_q/2}$. Thus $b \ll \alpha_q \ll 1$ at the full cap
learns the common signal and attains $r/2$.

If instead $\alpha_q \to \alpha_0 \in (0,2]$ and $k_qb \to 0$, the rare moment
tends to $r$. Along subsequences with $k_q = K$ fixed or $k_q \to \infty$, the
common moment tends, respectively, to
$A(\alpha_0)^KG_0 + \{1 - A(\alpha_0)^K\}g(\alpha_0)$ or $g(\alpha_0)$. Both
give risk at least $r/2$. These cases explain why substituting time zero into
the positive-time limit alone would miss some attainable transient risks.

\subsection{Reduction of the limiting trajectories}

Set $\Psi_U(\alpha,t) = \{g(\alpha) + H^\circ_\alpha(t)\}/2$. Both $\Psi_A$ are
smooth through $\alpha = 0,2$ when defined by their differential equations, and
$\Psi_U(0,t) = r/2$. For $0 < \alpha < 2$, each trajectory is monotone in time.
An increasing S trajectory stays above $v/2$, and an increasing U trajectory
stays above $(g + r)/2 \geq r/2$. The respective baselines are the values at
rate zero in the limiting problems. A decreasing trajectory reaches its minimum
at the full cap $T$.

Rates $\alpha \in (1,2)$ are dominated by $2 - \alpha \in (0,1)$ at the same
time: the decay $d(\alpha)$ is unchanged, while each noise term
\[
  g(\alpha),\qquad \frac{\nu\alpha}{2 - \alpha},\qquad \frac{\alpha\{\nu + g(\alpha)\}}{2 - \alpha}
\]
is nondecreasing in $\alpha$. At rate two the risks are $(v + 4\nu t)/2$ and
$\{g(2) + r + 4(\nu + g(2))t\}/2$, also above the baselines. Consequently,
\begin{equation}
  \Psi_A(\alpha,t) \geq L_A(T), \qquad 0
                   \leq \alpha
                   \leq 2,\quad 0
                   \leq t
                   \leq T. \label{eq:limit-trajectory-reduction}
\end{equation}
On $[0,3] \times[0,T_1]$, both flows and their rate derivatives are bounded, as
follows by differentiating their finite-dimensional matrix exponentials. Since
$c_A(b) - 2 = O(b)$, replacing a retained rate above two by two changes its
limiting risk by only $O(b)$.

\subsection{Lower bounds and recovery over the full oracle domain}

Combining \cref{eq:limit-s-uniform,eq:limit-trajectory-reduction} gives
$R^S_{R,k}(\alpha) \geq L_S(T) - O(b)$ at every retained rate and stop. Above
$c_S(b)$, the exact bound gives $R^S_{R,k} \geq wv/2 \geq L_S(T) - 2bv$. The
same lower bound therefore holds over the entire S oracle domain.

For U, divide the retained rate-and-stop choices according to whether $k\alpha$
reaches $J_b = 2\log(1/b)$. If $k\alpha \geq J_b$, then
$A^k \leq e^{-k\alpha/2} \leq b$, so
\cref{eq:limit-u-uniform,eq:limit-trajectory-reduction} give
$R^U_{R,k}(\alpha) \geq L_U(T) - O(b)$. If $k\alpha < J_b$, take $b$ small
enough that $1 - 2b\alpha \geq 0$. Positivity gives
\[
  H_k \geq D_b^kr \geq (1 - 2b\alpha)^kr \geq r(1 - 2bk\alpha).
\]
It follows that
\[
  R^U_{R,k}(\alpha) \geq \frac r2 - 2br - rbJ_b
                    \geq L_U(T) - 2br - 2rb\log(1/b).
\]
Above $c_U(b)$, the bound $R^U_{R,k} \geq wr/2 \geq L_U(T) - 2br$ suffices. Thus
the U oracle is bounded below by $L_U(T) - O(b\log(1/b))$. These bounds permit
taking infima directly, without requiring finite-prevalence minimizers to exist
or their rates and stopping times to converge.

For matching upper bounds, let $\alpha_A(T) \in[0,1]$ minimize $F_A$. Continuity
ensures existence. For S, use $\alpha_S(T)$ at $k = M_q$;
\cref{eq:limit-s-uniform} gives risk $L_S(T) + O(b)$. For U, use the full cap
and the rate
\[
  \alpha_q = \max\left\{ \alpha_U(T),\frac{2b\log(1/b)}{T} \right\}.
\]
This lies in $[0,1]$ for small $b$ and ensures $M_q\alpha_q \geq J_b$. The
common transient is $O(b)$, and $F_U$ has a bounded rate derivative on
$[0,1] \times[T_0,T_1]$. Hence
\[
  R^U_{R,M_q}(\alpha_q) = F_U(\alpha_q;T) + O(b) \leq L_U(T) + O(b\log(1/b)).
\]
This construction includes zero or vanishing minimizing rates and zero signal or
noise. Together with the lower bounds, it proves \cref{eq:oracle-limit-errors}.

Finally, bounding the time derivatives of the objectives uniformly in rate and
taking minima yields
\[
  |L_S(T) - L_S(\tau)| \leq \frac{v + \nu}{2}|T
                            - \tau|,\qquad |L_U(T)
                            - L_U(\tau)|
                       \leq \frac{r + 8\nu/7}{2}|T - \tau|.
\]
Since $T \to \tau \in (0,\infty)$, both oracle limits follow. Rounding $\tau/b$
to an integer introduces an additional budget discrepancy of only $O(b)$.

\paragraph{Why the positive-budget assumption matters.}

The theorem does not assert uniformity when signals, noise, or budgets diverge.
Nor can we drop the condition $\tau > 0$ without specifying the cap sequence.
For $s = 4,r = 0,\nu = 1$, exact one-step minimization gives
\[
  R_U^*(1) = \frac12 - \frac{1}{4(1 + 16bw)} \longrightarrow \frac14.
\]
If instead $M_q \to \infty$ and $M_qb \to 0$, taking $\alpha_q = M_q^{-1/2}$
gives risk tending to zero by the vanishing-rate bounds. Both sequences have
zero limiting rescaled budget, but different oracle limits.

\section{Variational decisions from the limiting risks}
\label{sec:oracle-variational}

\subsection{Unique optimal rates}

For $\nu > 0$, write $x = v/\nu$, $\rho = r/\nu$, and, for $\alpha \in[0,1]$,
\[
  p = \frac{\alpha}{2 - \alpha},\qquad h
    = \frac{\alpha}{8 - \alpha},\qquad P
    = p(1 + h),\qquad z
    = e^{-\tau d(\alpha)}.
\]
Then $L_S = \nu J_S/2$ and $L_U = \nu J_U/2$, where
\begin{equation}
  J_S(x,\tau)
  = \min_{[0,1]}\{xz + p(1 - z)\},\qquad J_U(\rho,\tau)
  = \min_{[0,1]}\{\rho z + h + P(1 - z)\}. \label{eq:dimensionless-oracles}
\end{equation}
Only the two signal-to-noise ratios and the rescaled budget determine the
preferred scaling. Multiplying $s,r,\nu$ by the same positive number scales both
risks by that number and leaves the optimal rates unchanged.

On $0 < \alpha < 1$, define
\[
  Q_S = p + \frac{p'(e^{\tau d} - 1)}{\tau d'},\qquad Q_U
      = P + \frac{h'e^{\tau d} + P'(e^{\tau d} - 1)}{\tau d'}, \qquad d'
      = 2(1 - \alpha).
\]
Primes denote rate derivatives. Both functions increase strictly: $p,h,P$ have
positive first and second derivatives, and $e^{\tau d}$ and $1/d'$ increase on
this interval. Their limits at the endpoints are
\[
  Q_S(0) = 0,\qquad Q_U(0) = \frac{1}{16\tau},\qquad Q_S(1-) = Q_U(1-) = \infty.
\]
The derivatives of the respective objectives in \cref{eq:dimensionless-oracles}
are $\tau d'z(Q_S - x)$ and $\tau d'z(Q_U - \rho)$. Thus, if $v,\nu > 0$, S has
one minimizer in $(0,1)$, given by $Q_S = x$. U has one minimizer in $(0,1)$
when $r\tau > \nu/16$, given by $Q_U = \rho$, and otherwise its unique minimizer
is zero, including at equality. If $v = 0$, S also minimizes uniquely at zero
for $\nu > 0$. If $\nu = 0$, choosing rate one gives $L_S = ve^{-\tau}/2$ and
$L_U = re^{-\tau}/2$; a zero signal makes the corresponding objective
identically zero.

The numerical solver in the source supplement brackets these roots and
multiplies the derivative by $z$ to avoid evaluating growing exponentials. It
evaluates the proved characterization in floating-point arithmetic; its answers
are not interval certificates.

\subsection{Allocation boundaries}

\Cref{cor:oracle-allocation} holds the total initial signal contribution $v$
fixed. Its boundary can be written
\begin{equation}
  \theta_* = \frac1x\max_{0 \leq \alpha \leq 1} \left[ \{J_S(x,\tau) - h - P(1 - z)\}e^{\tau d} \right].
\end{equation}
For each rate, the bracketed expression is the rare signal-to-noise ratio at
which that U trajectory equals $J_S$. U wins if at least one trajectory lies
below $J_S$, which explains the maximum. The bounds on the slope used in the
corollary also give
\[
  \frac{e^{-\tau}}2(r_2 - r_1) \leq L_U(r_2) - L_U(r_1) \leq \frac12(r_2 - r_1).
\]

An equivalent comparison fixes $r,\nu,\tau > 0$ and varies $s$. U's limiting
risk is constant, while
\[
  \frac{e^{-\tau}}8(s_2 - s_1) \leq L_S(s_2) - L_S(s_1) \leq \frac18(s_2 - s_1).
\]
At $s = 0$, S wins by the same endpoint argument as in
\cref{cor:oracle-allocation}; its risk then grows without bound as $s$
increases. There is therefore exactly one $s_* > 0$ giving equality, with S
preferred below it and U above it. Its variational expression and a bracket are
\[
  s_*
  =    4\nu\max_{0 \leq \alpha \leq 1} \left[ \{J_U(\rho,\tau) - p\}e^{\tau d} + p - \rho \right], \qquad 0
  <    s_*
  \leq 4r(e^\tau - 1).
\]
The bracket follows from $L_S \geq ve^{-\tau}/2$ and $L_U \leq r/2$. Either
boundary can also be computed by scalar bisection of the optimized risk
difference. Uniqueness does not ensure good conditioning when $\tau$ is large.
These results concern the stated allocation paths and do not assert a unique
transition under every constraint on signal strength.

\subsection{Noise and budget regimes}
\label{sec:oracle-noise-budget}

For fixed signals and budget, both limiting risks are nondecreasing and concave
in $\nu$, because they are minima of affine functions with nonnegative noise
coefficients. Their difference need not be monotone. If $s > 0$, zero noise
gives $L_S - L_U = se^{-\tau}/8 > 0$, so sufficiently small positive noise also
favors U. Once $\nu \geq 16r\tau$, U has risk $r/2$. As $\nu \to \infty$, the
optimal S rate must tend to zero, since rates bounded away from zero have
unbounded noise cost. Its signal term tends to $v/2$, while initialization
provides the matching upper bound. Hence
\[
  L_S \longrightarrow v/2,\qquad L_U
      \longrightarrow r/2, \qquad L_S - L_U
      \longrightarrow s/8.
\]
U also wins at sufficiently large noise when $s > 0$. For positive common
signals in the rare-dominant family of \cref{thm:conference-predictive}, S wins
at noise variance $\nu = r$ and $\tau \in[1,4]$. Continuity then implies at
least two noise crossings. It does not imply exactly two, or crossings for every
signal pair. \Cref{tab:oracle-noise-regimes} illustrates these changes in
preferred scaling for $s = 0.1$, $r = 1$, and $\tau = 2$.

\begin{table}[htbp]
  \centering
  \begin{tabular}{rrrr}
    \toprule
    $\nu$ &    $L_U$ &    $L_S$ & Preferred scaling \\
    \midrule
        0 & 0.067668 & 0.069359 &                 U \\
      0.1 & 0.109877 & 0.103572 &                 S \\
        1 & 0.256267 & 0.232599 &                 S \\
      100 & 0.500000 & 0.502306 &                 U \\
    \bottomrule
  \end{tabular}
  \caption{
    Limiting oracle risks across noise variances for $s = 0.1$, $r = 1$, and
    $\tau = 2$. The preferred scaling has the smaller risk: U at the smallest
    and largest noise variances shown, and S at the intermediate values.
  }
  \label{tab:oracle-noise-regimes}
\end{table}

Each limiting oracle risk is nonincreasing in $\tau$: the trajectory reduction
allows the earlier stopping choice under every larger cap. For fixed $v,\nu > 0$
and $\tau \downarrow 0$, expanding uniformly on $\alpha \in[0,1]$ gives
\[
  L_S =               \frac v2
                      - \frac{\tau v^2}{2(v + \nu)}
                      + O(\tau^2),\qquad \alpha_S
      \longrightarrow \frac{v}{v + \nu}.
\]
The first-order reduction in twice the risk is
$\tau\{2\alpha v - \alpha^2(v + \nu)\}$, a strictly concave quadratic. For
sufficiently small $\tau$, U has risk $r/2$ by its endpoint threshold. Thus
small positive rescaled budgets favor U when $s > 0$. This statement takes the
prevalence limit first; the common coordinate has already been learned.

For fixed $r,\nu > 0$ and finite $s \geq 0$, the large-budget expansions are
\begin{equation}
  \begin{aligned}
    L_S & = \frac{\nu}{8\tau} \left[ 1 + \log\!\left( \frac{4\tau v}{\nu} \right) \right] + o(\tau^{-1}),     \\
    L_U & = \frac{5\nu}{32\tau} \left[ 1 + \log\!\left( \frac{16\tau r}{5\nu} \right) \right] + o(\tau^{-1}).
  \end{aligned} \label{eq:large-budget-oracles}
\end{equation}
To justify them, the candidate $\alpha = (\log\tau)/(2\tau)$ gives risk
$O((\log\tau)/\tau)$. At any minimizer, the positive signal term forces
$z = O((\log\tau)/\tau)$. Under S, the noise term $\nu p(1 - z)$ then bounds
$\alpha$ by the same order; under U, $\nu h$ does so directly. Thus
$\tau\alpha^2 \to 0$. Uniformly in this minimizing range, twice the risks are
\[
  ve^{-2\tau\alpha}
  + \frac{\nu\alpha}{2}
  + o(\tau^{-1}),\qquad re^{-2\tau\alpha}
  + \frac{5\nu\alpha}{8}
  + o(\tau^{-1}).
\]
The remainders are $O((\log\tau)^3/\tau^2)$ at both the exact minimizers and the
minimizers of these approximations. Minimizing $Be^{-2\tau\alpha} + c\alpha$
gives $\alpha = \log(2\tau B/c)/(2\tau)$ for large $\tau$, proving
\cref{eq:large-budget-oracles}. U pays $\nu\alpha/8$ in common-coordinate noise
in addition to $\nu\alpha/2$ in rare-coordinate noise, while S has only the
latter at leading order. Consequently $L_S/L_U \to 4/5$.

For $s > 0$, the small- and large-budget signs imply at least one budget
crossing, without asserting uniqueness. At $s = r = \nu = 1$, numerical risks
favor U at $\tau = 1$ and S at $\tau = 10$. Convergence can be slow: for
$s = 10,r = 0.1,\nu = 2$, U still wins at $\tau = 10^5$, while S wins at
$\tau = 10^7$. The assumptions matter: with $\nu = 0$, U is always weakly
better; with $r = 0$, its limiting risk is zero; with $s = 0,r,\nu > 0$, S wins
at every finite positive budget.

The deterministic checks in the source supplement test the moment
approximations, bounds on unrestricted infima, scalar roots, boundaries, and
expansions. They supplement the proofs and are separate from the
finite-prevalence certificates.

\section{One rare feature and several common features}
\label{sec:common-block-details}

\subsection{Moment reduction}

Let $I_1,\ldots,I_m$ be independent Bernoulli variables with prevalence $1/2$,
and let $J$ be an independent Bernoulli variable with prevalence $q$, where
$m \geq 1$ and $0 < q < 1/2$. With population-centered features, the response is
\[
  y = \sum_{j=1}^m\beta_j^*(I_j - 1/2)
      + \beta_R^*(J - q)
      + \varepsilon, \qquad \E(\varepsilon \mid x)
    = 0, \qquad \E(\varepsilon^2 \mid x)
    = \sigma^2.
\]
Training uses fresh observations, zero initialization, and a fixed zero
intercept. On the original coefficient scale, the rare-coordinate step is
$\alpha$ and every common-coordinate step is $4u$, with $u = \alpha/4$ for U or
$u = b\alpha$ for S. Here $b = q(1 - q)$, as before. Define
$s = \sum_{j=1}^m(\beta_j^*)^2$, $r = (\beta_R^*)^2$, and
\[
  V_{j,k} = \E[(\beta_{j,k} - \beta_j^*)^2],\qquad G_k
          = \frac14\sum_{j=1}^m V_{j,k},\qquad H_k
          = \E[(\beta_{R,k} - \beta_R^*)^2].
\]

The combined moments satisfy
\begin{equation}
  \begin{aligned}
    G_{k+1} & = (1 - 2u + mu^2)G_k + mu^2(bH_k + \sigma^2), \\
    H_{k+1} & = D_bH_k + b\alpha^2(G_k + \sigma^2),
  \end{aligned} \label{eq:common-block-moments}
\end{equation}
with $(G_0,H_0) = (s/4,r)$ and $D_b = 1 - b\alpha(2 - \alpha + 3b\alpha)$. The
factors of $m$ show how additional common features increase the noise cost, even
when their true coefficients are zero.

Expanding the squared SGD update for a common coefficient gives
\[
  V_{j,k+1} = (1 - u)^2V_{j,k}
              + u^2\sum_{\substack{\ell=1\\\ell\ne j}}^m V_{\ell,k}
              + 4u^2bH_k
              + 4u^2\sigma^2.
\]
The common feature has second moment $1/4$ and fourth moment $1/16$.
Independence and centering remove terms involving two distinct coefficient
errors from this diagonal equation, even when those errors are correlated.
Summing over $j$ and dividing by four gives the first line of
\cref{eq:common-block-moments}. The rare equation follows from
$\E[(J - q)^4] = b(1 - 3b)$ and the same expansion. Conditioning the test
distribution on $J = 1$ leaves the common features centered and independent, so
the conditional second-moment matrix is diagonal and
$R_{R,k} = \{G_k + (1 - q)^2H_k\}/2$.

The recurrence depends on the common coefficients only through $s$. Thus the
expected rare risk is unchanged if the same squared signal is concentrated in
one common feature or spread across the block. Coefficient signs also have no
effect on this risk. Setting $m = 1$ recovers \cref{eq:predictive-moments}, with
$G_k = C_k/4$ and with the observation-noise terms multiplied by $\sigma^2$.

\subsection{Comparison across dimensions}

For fixed $q,s,r,\sigma^2$, write $R_A^*(M;m)$ for the infimum of rare risk
under scaling $A$ over every nonnegative constant rate and deterministic stop
$k \in \{0,\ldots,M\}$ in the model with $m$ common features.

\begin{proposition}[Increasing the common-feature dimension]
  For either scaling $A \in \{U,S\}$ and integers $m_2 \geq m_1 \geq 1$,
  \[
    R_A^*(M;m_2) \geq R_A^*(M;m_1).
  \]
  The same inequality holds at each fixed nonnegative rare-coordinate step
  $\alpha$ and stop $k$.
\end{proposition}

\begin{proof}
  At fixed $\alpha$, the value of $u$ does not depend on $m$. Every coefficient
  in \cref{eq:common-block-moments} is nonnegative. In particular,
  \[
    1 - 2u + mu^2 =    (1 - u)^2 + (m - 1)u^2
                  \geq 0, \qquad D_b
                  =    (1 - b\alpha)^2 + b\alpha^2(1 - 4b)
                  \geq 0.
  \]
  Increasing $m$ increases the coefficients and noise term in the common
  equation and leaves the rare equation unchanged. Both systems start at
  $(s/4,r)$, so induction gives coordinatewise larger moments at every stop. The
  risk weights are nonnegative and independent of $m$, proving the fixed-rate
  comparison. Taking the infimum over the same rates and stops proves the oracle
  comparison.
\end{proof}

The proposition applies separately to each method. It does not imply that their
risk difference is monotone in $m$, or that either ordering from the two-feature
theorem persists.

\subsection{Numerical comparisons}
\label{sec:common-block-numerics}

Both orderings persist in all 126 numerical comparisons over seven dimensions,
three prevalences, three caps, and the two signal allocations. Each scaling
searches its own finite grid of stable rates and every integer stop within the
cap. \Cref{fig:common-block} shows the slice with roughly two expected rare
observations. When the rare signal dominates and $q = 10^{-4}$, increasing $m$
from one to 64 raises U's selected risk from $0.256$ to $0.500$, while SD's
changes from $0.229$ to $0.230$. When the common signal dominates, U retains its
advantage, though that advantage narrows with dimension, especially at the
larger prevalence. These comparisons extend the numerical evidence to more
features; they do not extend the uniform gaps in
\cref{thm:conference-predictive}.

\begin{figure}[t]
  \centering
  \includegraphics{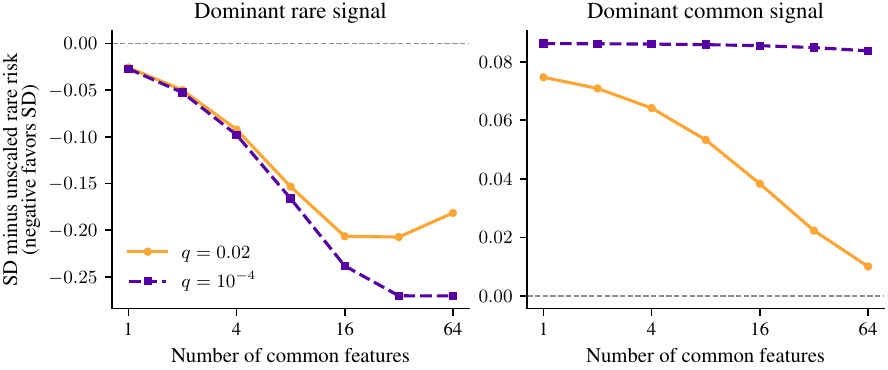}
  \caption{
    Additional common features change the cost of normalization at fixed signal
    strength. The sums of squared common and rare coefficients are
    $(s,r) = (0.01,1)$ (left) or $(1,0.01)$ (right), with noise variance one.
    Caps are $M = 102$ at $q = 0.02$ and $M = 20{,}002$ at $q = 10^{-4}$, giving
    $Mq(1 - q) \approx 2$. Curves show SD-minus-U rare risk after separate
    searches over finite stable-rate grids and all integer stops; negative
    values favor SD\@. Each panel has its own vertical scale.
    \Cref{sec:common-block-numerics} gives all settings and checks.
  }
  \label{fig:common-block}
\end{figure}

We use $m \in \{1,2,4,8,16,32,64\}$, $q \in \{0.02,0.005,10^{-4}\}$,
$\sigma^2 = 1$, and $(s,r) \in \{(0.01,1),\allowbreak(1,0.01)\}$. The
observation caps are
\[
  \begin{array}{c|rrr}
       q    &     \multicolumn{3}{c}{M}      \\
    \hline
     0.02   &       52 &      102 &      204 \\
     0.005  &      202 &      402 &      804 \\
    10^{-4} & 10{,}002 & 20{,}002 & 40{,}004
  \end{array}
\]
These caps give values of $Mq(1 - q)$ close to $1,2,4$, all within $[1,4]$. The
initial values $(G_0,H_0) = (s/4,r)$ are the same across dimensions. The study
uses exact expected moments evaluated in floating point; it does not simulate
training or estimate risks from a test sample.

Each scaling uses zero rate and 1,025 positive rates geometrically spaced
between $10^{-7}$ and $0.999$ times its mean-square stability limit. To compute
that limit, write $u = c\alpha$, where $c = 1/4$ for U or $c = b$ for S, and let
$T$ be the homogeneous transition in \cref{eq:common-block-moments}. Then
\[
  \det(I - T) = bc\alpha^2(4 - L\alpha + K\alpha^2),\qquad L
              = 2(1 - 3b + mc),\qquad K
              = mc(1 - 4b).
\]
The first positive root gives
\[
  \alpha_{\mathrm{stab}} = \frac{8}{L + \sqrt{L^2 - 16K}}.
\]
The search includes every integer stop from zero through each cap. It selects
the rate and stop that minimize expected rare risk, before any training
realization is observed. This finite stable-rate search is narrower than the
unrestricted oracle in \cref{thm:conference-predictive}.

Both signal orderings persist in all 126 paired comparisons. The smallest
selected gaps are $0.01959$ in favor of SD for the dominant rare signal and
$0.00633$ in favor of U for the dominant common signal. A nested coarse search
retains alternate positive rates and zero, giving 513 positive rates. Refining
this grid changes individual selected risks by at most $3.50 \times 10^{-5}$ and
preserves every ordering. No selected rate at a positive stop lies at either
endpoint of the positive rate grid. These checks measure sensitivity to the
tested grid; they do not certify its distance from the continuous optimum.

\begin{figure}[tp]
  \centering
  \includegraphics{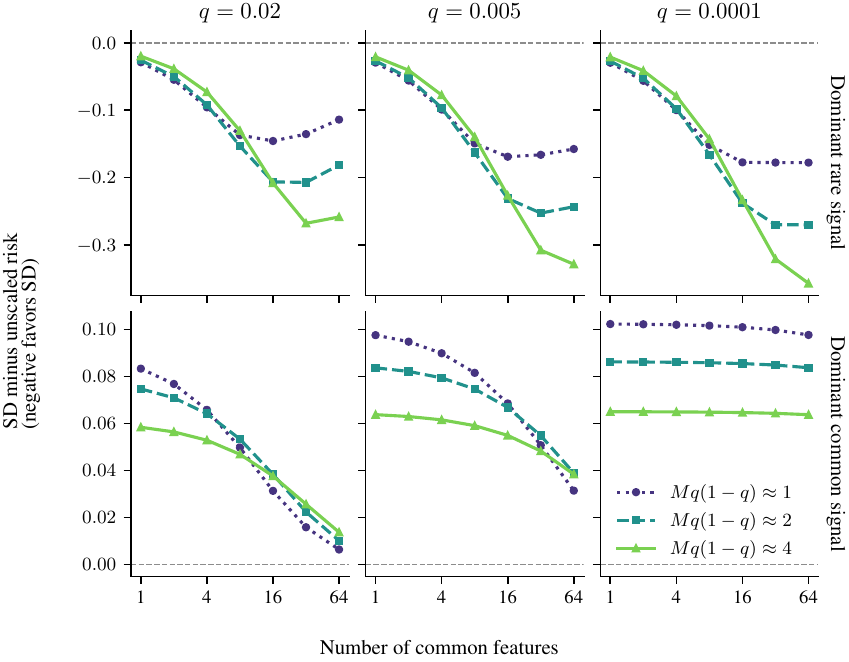}
  \caption{
    Dimension comparisons across all three prevalences and budgets. The top row
    uses $(s,r) = (0.01,1)$ and the bottom row uses $(1,0.01)$, holding total
    squared signal fixed within each row. Negative SD-minus-U rare risk favors
    standardization. Each scaling searches its own finite grid of stable rates
    and every integer stop through the cap. The rows have different vertical
    scales. At the larger prevalences, the size of SD's rare-signal advantage
    can decrease after initially increasing with dimension.
  }
  \label{fig:common-block-full}
\end{figure}

\Cref{fig:common-block-full} shows all slices. At $q = 0.02$ and $M = 102$,
increasing $m$ from one to 64 raises U's common-signal risk from $0.01578$ to
$0.09894$, while SD's rises from $0.09049$ to $0.10901$. Thus U's advantage
narrows as both risks increase. At $q = 10^{-4}$, SD's rare-signal risk changes
little over the tested dimensions, while U's increases substantially, as shown
in \cref{fig:common-block}.

An independent checker verifies 60 one-step reductions by enumerating individual
observations with rational arithmetic in dimensions two through six. These cases
include nonzero cross moments of the coefficient errors and rates above the
stability range. Further checks compare full second-moment tensor powers,
concentrated and spread common signals, and a small rate-and-stop search with
independent observation-level calculations. All 252 selected risks are replayed
with matrix powers rather than sequential updates; the largest discrepancy is
less than $8 \times 10^{-13}$.

The source supplement retains each rate's lowest risk and the earliest stop that
attains it at every cap, together with the grids, selected rates and stops, and
refinement diagnostics. The reproduction guide explains how to regenerate the
study and verify the selected results.

\section{Numerical illustrations}
\label{sec:predictive-intro-figure}

\subsection{Oracle risk and its components}

\Cref{fig:conference-tuned-intro} illustrates the two signal families in
\cref{thm:conference-predictive} at $q = 10^{-4}$ and $\sigma^2 = 1$. We set
$(\beta_C^*,\beta_R^*) = (0.1,1)$ for the rare-signal panel and $(1,0.1)$ for
the common-signal panel. All other assumptions match the theorem. We evaluate 61
integer observation caps, spaced evenly before rounding from $10{,}002$ to
$40{,}004$, so $1 \leq Mq(1 - q) \leq 4$. The horizontal axis reports the
expected number of rare observations at the cap, $Mq$, rather than a realized
count.

For each rule, we use the parameterization in \cref{eq:conference-oracle-rates}
and a grid consisting of zero and 4097 geometrically spaced positive values of
$\alpha$ from $10^{-6}$ to $2$. We iterate the exact
recurrence~\labelcref{eq:predictive-moments}, including both sources of training
variance. At each cap, we select the rate and integer stopping time with the
lowest population rare risk, searching every stop from zero through the cap. The
curves connect these numerical minima; they are not trajectories of a single
trained model. Selection uses population risk averaged over training randomness,
rather than validation observations.

We repeat the search with every second positive grid point and check all 244
selected endpoints against a separate full tensor moment operator constructed by
enumerating the four binary observations. The reproduction script requires the
maximum risk change on grid refinement to be below $2 \times 10^{-6}$ and the
maximum tensor replay discrepancy to be below $10^{-10}$. All selected rates are
interior to the grid. These checks assess numerical accuracy within the stated
search; the theorem supplies the separate comparison over all nonnegative
constant rates. The figure illustrates an existing theoretical result and is not
a prospective evaluation of its predictions.

\Cref{fig:tuned-decomposition} uses the same selections at cap $M = 20{,}503$.
We compute the squared mean error, noiseless feature-sampling covariance, and
observation-noise covariance with the full moment recurrence, then weight each
coordinate by its contribution to rare risk. The saved summary retains these
components for every selected endpoint. We check each component against the
separately enumerated tensor operator and require its discrepancy, and that
between the component sum and the selected risk, to be below $10^{-10}$.

\subsubsection{Variance decomposition}
\label{sec:variance-decomposition}

Write $D = \eta P$ and $C = \E(xx^\top)$. To prove
\cref{eq:conference-decomposition}, we run the noiseless recursion on the same
feature stream. Its mean error is $e_k = -(I - DC)^k\beta^*$ and its covariance
is $A_k \succeq 0$. Conditional on all sampled features, the component due to
noise has mean zero, so its cross moment with the noiseless component vanishes.
Its covariance is $\sigma^2N_k \succeq 0$, which proves the decomposition.
Equivalently, if
$\mathcal T_D(V) = \E\bigl((I - Dxx^\top)V(I - xx^\top D)\bigr)$, then
\[
  V_k = \mathcal T_D^k(\beta^*\beta^{*\top})
        + \sigma^2\sum_{j=0}^{k-1}\mathcal T_D^j(DCD).
\]
The first term is the noiseless second moment $e_ke_k^\top + A_k$, and the
second is the observation-noise covariance $\sigma^2N_k$.

\subsection{Finite-prevalence comparisons with the limiting risks}
\label{sec:oracle-finite}

We compare \cref{thm:oracle-limits,cor:oracle-allocation} with exact moments
under the same population model. The saved study contains 43 combinations of
signal, noise, and target budget at six prevalences from $0.04$ to
$2 \times 10^{-6}$. The settings include both signal pairs from the
finite-prevalence theorem at $\tau = 1,2,4$; common signals at 90\%, 100\%, and
110\% of the predicted crossing; eight allocations with $s/4 + r = 1$ and
$\tau = 2$; four noise levels; points around U's rare-learning threshold;
budgets from $0.01$ to $1000$; and zero-signal and zero-noise endpoints. These
are selected slices rather than uniform coverage of the surrounding parameter
ranges. \Cref{fig:oracle-finite} summarizes the comparisons.

Each scaling minimizes expected rare risk over separate finite grids. The fine
rate grid contains 2049 fractions of its mean-square stability limit, combining
1025 linear and 1025 logarithmic points, with zero and positive fractions from
$10^{-10}$ to $1 - 10^{-6}$. Limiting optimal rates are not inserted into the
grid, and rates outside the stable interval are not searched. Stopping grids
contain initialization, every integer through 64, all caps, and rounded
geometric checkpoints up to the largest cap. These grids have 663 to 883
distinct stops, depending on prevalence. The largest cap is $500{,}001{,}001$.
Matrix powers of the affine recurrence evaluate every candidate without Monte
Carlo noise.

Caps round $\tau/b$ upward, except at target $\tau = 4$, where they round
downward to stay within the finite-prevalence theorem's range. Primary errors
use $L_A(Mb)$, and the saved results also retain predictions at the nominal
target. At $q = 0.04$, target $\tau = 0.01$ gives one observation and actual
$Mb = 0.0384$, so the coarse-budget discrepancy includes appreciable rounding.

Coarse grids take every second point used to generate the rate and geometric
stopping grids, retaining the early integer stops and all caps. Separate rate
refinement, stop refinement, and joint refinement assess search resolution. The
largest decrease under joint refinement is $1.362 \times 10^{-5}$. Stop
refinement alone changes no selected risk: every positive-stop selection uses
its cap, and the others use initialization. No selected rate at a positive stop
reaches an edge of the rate grid. These observations do not bound the error
relative to the unrestricted oracle or prove that full-cap training is always
optimal at finite prevalence.

\begin{figure}[t]
  \centering
  \subfigure[Largest absolute risk discrepancy from $L_A(Mb)$ across all 43 settings at each prevalence.]{%
    \includegraphics{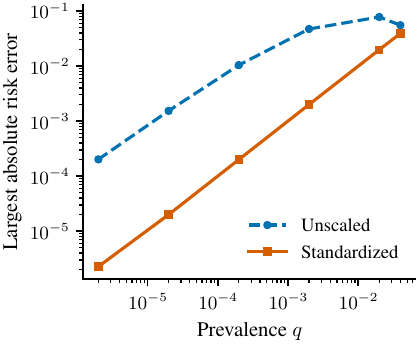}%
  }\hfill
  \subfigure[Allocation at fixed initial limiting subgroup risk,
    $s/4 + r = 1$ and $\tau \approx 2$.]{%
    \includegraphics{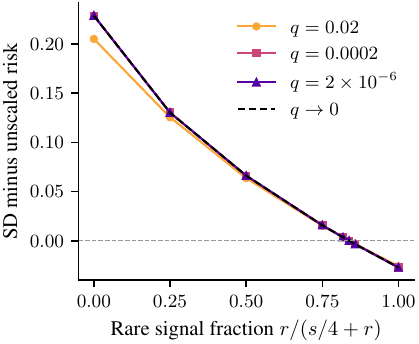}%
  }\par\medskip
  \subfigure[Noise can change the preferred scaling twice,
    with a finite-prevalence disagreement at high noise; $s = 0.1$, $r = 1$,
    and $\tau \approx 2$.]{%
    \includegraphics{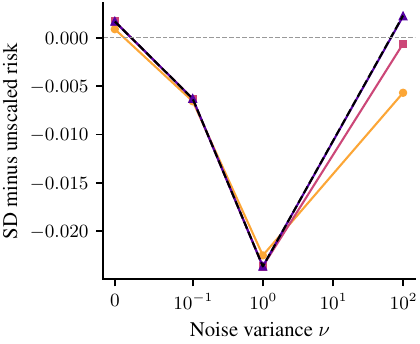}%
  }\hfill
  \subfigure[The budget comparison approaches the large-budget risk ratio $4/5$ obtained after taking the prevalence limit; $s = r = \nu = 1$.]{%
    \includegraphics{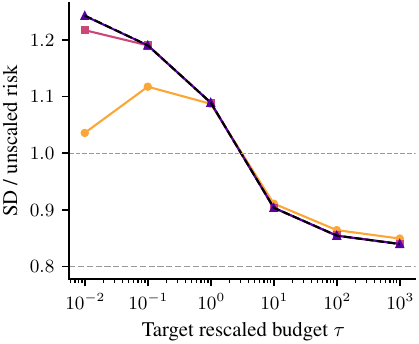}%
  }
  \caption{
    Exact finite-grid selections approach the limiting oracle predictions.
    Colored curves connect finite-grid results; black dashed curves connect
    limiting predictions at the sampled settings. Panels (b)--(d) use nominal
    target budgets, so coarse caps also reflect rounding. Lines do not certify
    intervening crossing locations.
  }
  \label{fig:oracle-finite}
\end{figure}

Both theorem orderings hold in every sampled comparison using the theorem's
signal pairs. At $q = 2 \times 10^{-6}$, the largest errors for those signal
pairs are $2.986 \times 10^{-6}$ for U and $1.535 \times 10^{-6}$ for S. Across
all 43 settings, the corresponding maxima are $2.01631 \times 10^{-4}$ and
$2.27583 \times 10^{-6}$; the largest U error occurs at target $\tau = 0.01$.
Every setting with a nonzero nominal limiting gap has the predicted sign at the
smallest prevalence. The sampled common-signal crossings are bracketed by 90\%
and 110\% of their limiting values, and the allocation slice recovers the
ordering on both sides of $\theta_* \approx 0.83770278$. Cases exactly at
limiting ties remain sensitive to rounding and search resolution and do not
locate a finite-prevalence boundary.

The U endpoint has a concrete finite-prevalence interpretation. At
$\tau = 2,s = 1,r = 0.01,\nu = 1$, the limiting rate is zero and
$L_U = r/2 = 0.005$. The selected finite rate nevertheless remains positive,
falling from $0.07003$ at $q = 0.02$ to $2.866 \times 10^{-5}$ at
$q = 2 \times 10^{-6}$, while the risk falls from $0.01570$ to $0.00500127$.
These rates learn the common signal; literal zero-rate training would retain its
error.

\paragraph{A finite-prevalence disagreement.}

For $s = 0.1,r = 1$, target $\tau = 2$, and $\nu = 100$, the limiting risks
favor U: $L_U = 0.500000$ and $L_S \approx 0.502306$. The finite-grid SD-minus-U
gaps are instead $-0.00568756$ at $q = 0.02$, $-0.00563282$ at $q = 0.002$, and
$-0.000608220$ at $q = 0.0002$. They become positive at $q = 0.00002$ and
$q = 0.000002$, with gaps $0.00180198$ and $0.00223702$. These discrepancies
exceed the observed refinement changes, showing that the limiting decision can
be inaccurate at finite prevalence. They do not certify a reversal of the
unrestricted finite-prevalence oracles.

At $s = r = \nu = 1$ and the smallest prevalence, U wins at target $\tau = 1$,
while S wins at $\tau = 10,100,1000$. At $\tau = 1000$, the finite risk ratio is
$0.839843$ and the limiting ratio is $0.839507$, both still above $4/5$. These
comparisons preserve the distinction between the prevalence limit and the
subsequent large-budget limit.

\paragraph{Reproduction and independent replay.}

The source supplement retains the settings, grids, all four searches, selected
risks, rates, stops, and summaries. An independent checker constructs
observation-level tensor updates and tests all four nested searches on a small
problem. It also replays all 2064 saved selections with full tensor updates and
long-double matrix powers; the recorded maximum discrepancy is
$7.38 \times 10^{-11}$. The accompanying reproduction guide explains how to
regenerate the searches, replay the selections, and plot the saved results.

\subsection{Introductory signal-allocation illustration}
\label{sec:signal-allocation}

\Cref{fig:intro-signal-crossing} uses the allocation path in
\cref{cor:oracle-allocation}: $s = 4(1 - \theta)$ and $r = \theta$, with
$\nu = 1$ and $\tau = 2$. Thus $s/4 + r = 1$ and the limiting subgroup risk at
initialization stays at $1/2$. The middle tick, $\theta = 1/2$, gives equal
common and rare contributions to this risk. We plot the scalar minima in
\cref{eq:limiting-objectives} along the allocation path. The eight allocation
values are $0$, $1/4$, $1/2$, $3/4$, $1$, and $\theta_* - 0.02$, $\theta_*$,
$\theta_* + 0.02$, where scalar root finding gives $\theta_* \approx 0.8377$.
Lines connect these saved values; \cref{cor:oracle-allocation} supplies the
unique-crossing guarantee. Unlike the finite-prevalence curves in
\cref{fig:oracle-finite}(b), this panel uses the limiting risks at exactly
$\tau = 2$, without integer-cap rounding.

The source supplement retains the sampled limiting risks and scalar-solver
checks. The reproduction guide explains how to check and plot these values
without new optimization or training.

\subsection{Paths to a shared least-squares solution}

\Cref{fig:intro-coefficient-paths} uses a constructed fixed dataset of 200
observations. The four binary feature combinations
$(I,J) = (0,0),\allowbreak(0,1),\allowbreak(1,0),\allowbreak(1,1)$ have counts
$98,2,98,2$, respectively. We center the columns by $(1/2,0.02)$, set
$y = x_C + x_R$ without noise, and fix the centered intercept at zero. Thus
$X^\top X/200 = \diag(1/4,0.0196)$ and the unique OLS solution is $(1,1)$. Both
methods initialize their coefficients at zero and use the full gradient of
empirical half squared loss. The unscaled learning rate is $0.4$, and the
standardized learning rate is $0.1$, so both common-coordinate errors contract
by $0.9$ per update.

We mark update 20 for both methods without selecting the stop. Their common
coefficients are both about $0.878$, while their rare coefficients are about
$0.146$ for unscaled training and $0.878$ for standardization. The pale paths
continue through update 2000 toward OLS\@. The reproduction script checks the
paths against the diagonal closed form, decreasing training loss, and
convergence to the common solution. This example illustrates how scaling changes
learning before stopping. The oracle theorem separately compares noisy
fresh-sample SGD after each method chooses its rate and stop; distance to a
sample OLS solution alone does not establish the ordering of oracle risks.

\section{Empirical methods}
\label{sec:empirical-methods}

The controlled study isolates changes to the oracle training model and compares
calibration with adaptive selection. All comparisons pair the scaling rules
within each replication. The protocols below specify which data and candidates
each comparison uses.

\subsection{Controlled departures from the oracle model}
\label{sec:theory-bridge-protocol}

The study in \cref{sec:theory-bridge} uses
$I \sim \operatorname{Bernoulli}(1/2)$ and $J \sim \operatorname{Bernoulli}(q)$
independently, with $q \in \{0.02,0.005\}$. Responses are
$Y = \beta_C^*(I - 1/2) + \beta_R^*(J - q) + \varepsilon$, where
$\varepsilon \sim N(0,1)$ is independent and the coefficient pairs are $(0.1,1)$
and $(1,0.1)$. Each update processes one observation. Caps are $(52,102,204)$ at
$q = 0.02$ and $(202,402,804)$ at $q = 0.005$, all satisfying
$1 \leq Mq(1 - q) \leq 4$. Settings were specified before examining the study's
outcomes; this is an exploratory extension of the existing theorem.

\subsubsection{Training arms and pairing}

\Cref{tab:bridge-arms} specifies the eight arms. The first 256 rows of an
independent feature and noise stream form the fixed training pool. Fresh SGD
uses successive rows of that stream; reuse samples pool indices independently
with replacement, retaining each row's original noise. Common random numbers
pair the two scalings, signal pairs, training arms, and candidate paths within
each replication. Estimated preprocessing uses the entire pool, including rows
beyond the SGD cap when $M < 256$. These arms therefore use additional data for
preprocessing, beyond the observations counted by the update cap.

\begin{table}[ht]
  \centering
  \small
  \begin{tabular}{clll}
    \toprule
    Arm & Sampling & Estimated preprocessing & Intercept                 \\
    \midrule
     1  & Fresh    & None                    & Fixed zero                \\
     2  & Fresh    & None                    & Fitted from zero          \\
     3  & Reuse    & None                    & Fixed zero                \\
     4  & Reuse    & Scales                  & Fixed zero                \\
     5  & Reuse    & None                    & Fitted from zero          \\
     6  & Reuse    & Scales                  & Fitted from zero          \\
     7  & Reuse    & Centering and scales    & Fitted from zero          \\
     8  & Reuse    & Centering and scales    & Fitted from response mean \\
    \bottomrule
  \end{tabular}
  \caption{
    Matched training arms, in the order shown in \cref{fig:bridge-controlled}.
    Centering and scales use population values unless listed as estimated.
    ``None'' denotes no estimated preprocessing. U uses divisor one in every
    arm. Every fitted intercept is unscaled and uses the same learning rate as
    the transformed coefficients.
  }
  \label{tab:bridge-arms}
\end{table}

All feature coefficients are initialized at zero. Empirical variances use
denominator 256, and a zero variance receives divisor one. We do not otherwise
freeze coefficients. Thus empirical centering automatically makes a constant
training column's gradient zero, while population centering can leave it
nonzero. Only arm 8 initializes the fitted intercept at the training response
mean. There are 21 training pools without rare observations at $q = 0.02$ and
1148 at $q = 0.005$, out of 4096 evaluation replications each.

\subsubsection{Candidate family and calibration}

Let $z(i,j)$ denote the transformed feature vector under a given arm and
scaling, augmented by a unit coordinate when fitting an intercept. Define
\[
  L = \max_{(i,j) \in \{0,1\}^2}\|z(i,j)\|^2, \qquad \eta = t/L.
\]
The rate factors are zero and 129 geometrically spaced values from $10^{-4}$ to
$1.9$. Each homogeneous update $I - \eta zz^\top$ is nonexpansive because
$0 \leq \eta\|z\|^2 \leq 1.9 < 2$. This condition rules out expansion by an
individual homogeneous update; it does not bound accumulated response noise or
guarantee good predictions on an unseen subgroup. It is a conservative search
restriction, not a characterization of every stable constant rate. For estimated
preprocessing, $L$ depends on the training pool. Calibration fixes $t$, while
each resulting trajectory retains one constant $\eta$.

Stopping checkpoints are the distinct integers obtained by rounding 65 geometric
points from one to the largest cap, together with zero and all three caps. A
nested coarse grid retains every second positive rate factor and every second
checkpoint, always including zero and the caps. Both scalings receive the same
number of trajectories. At cap $M$, the nominal cost of the search is $130M$
processed observations per arm, scaling, signal, and replication; caps and
selection policies share paths. Preprocessing, validation scoring, and
calibration require additional resources.

For each prevalence, 2048 calibration replications determine the mean exact
population rare risk of every candidate. Each arm, scaling, signal, and cap then
selects a rate factor and eligible stop. The fresh, fixed-intercept arm instead
uses exact expected moments for this choice. An independent set of 4096
evaluation replications assesses the fixed choices. The two stages use distinct
random streams derived from seed 20260924. This procedure estimates the optimum
within a finite candidate family, conditional on the calibration run, rather
than the unrestricted oracle in \cref{thm:conference-predictive}.

\begin{figure}[htbp]
  \centering
  \includegraphics[]{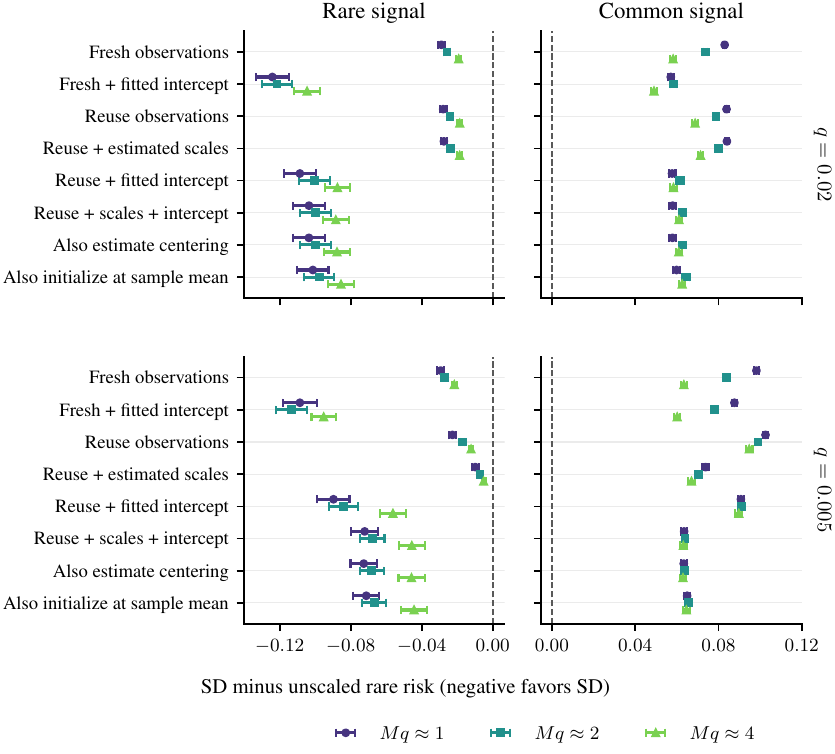}
  \caption{
    Both calibrated risk orderings persist across the eight training arms.
    Negative SD-minus-U rare risk favors SD\@. Colors distinguish the three
    caps; bars show two paired Monte Carlo standard errors over 4096 evaluation
    replications, conditional on the calibration choices. The last two rows
    successively add empirical centering and response-mean initialization to
    reuse with estimated scales and a fitted intercept. These comparisons
    concern two signal pairs and finite searches, not the full theorem families.
  }
  \label{fig:bridge-controlled}
\end{figure}

\subsubsection{Effects of changing training}

Both orderings persist in all 96 comparisons: SD has lower rare risk for the
dominant rare signal, and U has lower rare risk for the dominant common signal
(\cref{fig:bridge-controlled}). Each gap exceeds nine paired Monte Carlo
standard errors in magnitude, conditional on the calibration choices. A coarser
rate and stopping grid preserves all 96 signs. These results concern the two
tested signal pairs and do not extend the theorem to the changed training
procedures.

The size of the gap nevertheless depends on how training changes. For the
dominant rare signal at $q = 0.005$ and $M = 804$, the calibrated SD-minus-U gap
is $-0.0216$ under fresh training. Reuse changes the gap by $+0.00943$, and
estimating scales adds $+0.00710$, with paired standard errors $0.00093$ and
$0.00045$. Both changes narrow SD's advantage. Fitting an intercept with
population preprocessing instead changes the reuse gap by $-0.0440$ (standard
error $0.00345$), widening that advantage. Thus the preferred scaling can stay
the same while the cost of choosing the other scaling changes substantially.

\subsubsection{Sensitivity to calibration}
\label{sec:calibration-sensitivity}

We added a diagnostic after the original study to assess sensitivity to the
estimated tuning choices. For each prevalence, we randomly partition the 32
saved calibration shards into two groups of 16, each containing 1024
replications. We choose each rate factor and stop on one half and evaluate the
resulting mean rare risk on the other, then reverse the halves' roles. We repeat
this for 100 partitions using seed 20260928, retaining the same candidate family
and tie rule. Saved sums of candidate risks suffice for both selection and
evaluation; no training is repeated. The fresh fixed-intercept arm retains its
exact calibration and exact risk, leaving 84 comparisons with estimated tuning
choices.

Of these 84 comparisons, 83 retain the original ordering in all 200 splits. For
the dominant rare signal with reuse and estimated scales at $q = 0.005$ and
$M = 402$, SD wins in 198 splits and U in two. Its SD-minus-U gaps range from
$-0.01194$ to $+0.000687$, with median $-0.00595$. Across the 84 comparisons, at
least one scaling changes its selected rate or stop relative to full calibration
in 18 to 182 splits. The orderings are therefore more stable than the individual
tuning choices, although a small gap can change sign.

These overlapping splits give a descriptive sensitivity analysis at half the
original calibration size. They combine variation in calibration choices and
held-out risk estimates, and are not independent replications or confidence
intervals. They do not change the original 4096-replication evaluation or its
conditional error bars. The source supplement retains the partitions, choices,
held-out risks, and summary, with instructions for reconstructing and checking
the analysis from the saved calibration runs.

\subsubsection{Adaptive selection and exact evaluation}

Each evaluation replication also supplies an independent validation dataset with
2048 rows. Nested prefixes of sizes 128, 512, and 2048 select candidates by rare
validation half squared loss, using aggregate loss if the rare group is empty.
Ties favor an earlier stop, then a smaller rate factor. Empty rare validation
groups occur in $(306,1,0)$ replications at $q = 0.02$ and $(2206,326,0)$ at
$q = 0.005$, in increasing sample-size order. The validation selectors never
observe population risk.

\begin{figure}[htbp]
  \centering
  \includegraphics[]{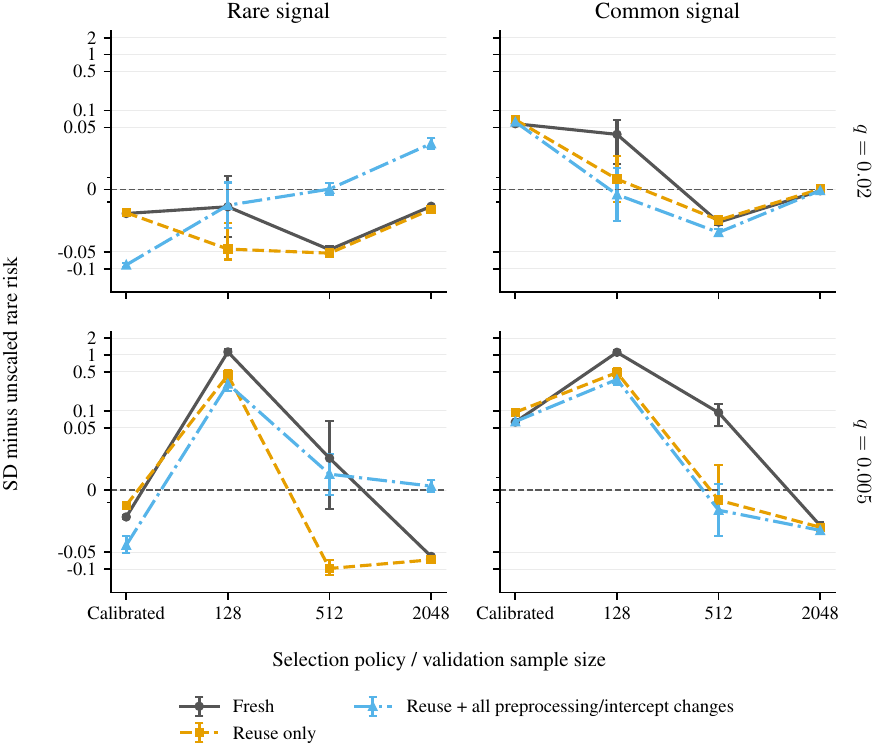}
  \caption{
    Changing selection on the same candidate paths at the largest cap for each
    prevalence. Calibrated choices are fixed before evaluation; the other points
    select by rare validation loss with the indicated sample size. Negative gaps
    favor SD\@. The blue curve includes all changes through arm 8. Bars show two
    paired Monte Carlo standard errors, conditional on the calibration choices.
    All panels share the same vertical scale and limits, with linear spacing
    between $-0.05$ and $0.05$ and logarithmic spacing outside this interval to
    keep both small and large gaps visible. Tick labels give the original risk
    differences. Lines connect discrete selection procedures and do not imply
    monotonic behavior with sample size.
  }
  \label{fig:bridge-validation}
\end{figure}

A further benchmark selects the fitted predictor with the lowest exact
population rare risk in each replication. To distinguish the two objectives, let
$r_R(\widehat f_{D,a})$ denote the population rare excess risk of candidate $a$
trained on random data and sampling choices $D$. For a fixed finite candidate
index set $\mathcal A$, including the allowed rate factors and stops,
\[
  \E_D\!\left[ \min_{a \in \mathcal A}r_R(\widehat f_{D,a}) \right] \leq \min_{a \in \mathcal A}\E_D\!\left[ r_R(\widehat f_{D,a}) \right].
\]

The left side allows adaptation to realized training randomness. The calibrated
procedure approximates the right side, with its choices learned on separate
replications. Neither side here is an unrestricted rate search, and the
reduction from the right side to the left can change the ordering between
scalings. The reversal under exact selection in \cref{tab:bridge-selection}
therefore concerns a different objective from the theorem, even without
validation noise.

\Cref{fig:bridge-validation} shows how validation selection changes the
comparisons across both prevalences and signals. Missing validation observations
create a further problem. For the fresh baseline with $q = 0.005$, $M = 804$,
and a dominant rare signal, validation with 128 rows changes the gap from
$-0.0216$ to $+1.133$ (standard error $0.062$). The rare group is absent in 2206
of 4096 validation sets. Conditional on its absence, the gap is $+2.193$; when
it is present, the gap is $-0.105$. Even within this conservative rate search,
validation can miss the group on which the two scalings differ most.

Evaluation enumerates the four binary cells analytically and excludes
irreducible test noise. For a predictor written in population-centered
coordinates as $a + b_C(I - 1/2) + b_R(J - q)$, its rare excess risk is
\[
  r_R = \frac{(b_C - \beta_C^*)^2}{8}
        + \frac{\{a + (1 - q)(b_R - \beta_R^*)\}^2}{2}.
\]
This includes the fitted intercept and any shift induced by estimated centering.
No prediction clipping is applied. We average paired SD-minus-U differences
across evaluation replications and report their Monte Carlo standard errors.
These errors condition on the frozen calibration choices; they exclude
calibration uncertainty and adjustments for multiple comparisons.

\subsubsection{Exact expected-risk calibration benchmark}
\label{sec:bridge-exact-calibration}

For the two fresh arms, a deterministic moment calculation identifies the best
fixed candidate in the same finite family used for \cref{tab:bridge-selection}.
Let $z$ be the population-centered, transformed feature vector, augmented by a
unit coordinate when fitting an intercept. If $e_k$ is the transformed
coefficient error after $k$ updates and $Q_k = \E[e_k e_k^\top]$, then
independent observations and unit noise variance give
\[
  Q_{k+1} = \E_z[(I - \eta zz^\top)Q_k(I - \eta zz^\top)]
            + \eta^2\E_z[zz^\top], \qquad Q_0
          = \theta^*\theta^{*\top}.
\]
Here $\theta^*$ is the true transformed coefficient vector, including a zero
intercept when fitted, and all coefficients start at zero. The expected rare
risk at stop $k$ is $\tfrac12\operatorname{tr}(Q_k\E[zz^\top\mid J=1])$. We sum
over the four feature cells and evaluate this exact recursion in floating-point
arithmetic at every saved rate and stopping checkpoint. Including the intercept
cross moments is essential. The optimization covers the saved finite family, not
all real rates and integer stops, and leaves the frozen calibration choices
intact.

For the fitted-intercept example in \cref{tab:bridge-selection}, the frozen
choices have expected rare risks $0.24385$ for U and $0.13823$ for SD\@. The
finite-family minima are $0.23184$ and $0.13823$, respectively, attained at stop
$204$ with rate factors $0.407566$ and $0.299562$. Calibration suboptimality is
therefore $0.01200$ for U and zero for SD\@. The optimal fixed-choice SD-minus-U
gap is $-0.09361$, so calibration suboptimality cannot account for the reversal
under selection. In the fixed-intercept baseline, the saved choices already
attain the finite-family minima, $0.15518$ for U and $0.13557$ for SD\@. These
expectations have no Monte Carlo error; the table estimates risks from 4096
evaluation runs. The source supplement retains expected risks for every
candidate, the minimizing choices, and excess risks due to calibration for both
fresh arms, both prevalences, all three caps, and both signals. Independent
short-path enumeration and the baseline's two-coordinate moment recursion check
the calculation. The reproduction guide gives commands to reconstruct it without
training.

\subsubsection{Distributions of selection gains}
\label{sec:bridge-gains}

For each scaling $s$ and evaluation run $D$, define the empirical selection gain
\[
  G_s(D) = r_R(\widehat f_{D,s,\mathrm{cal}})
           - r_R(\widehat f_{D,s,\mathrm{sel}}).
\]
Positive values indicate improvement over that scaling's frozen calibrated
choice. These gains describe selection among a fixed finite family of trained
predictors, conditional on the saved calibration choices. They do not estimate
an unrestricted oracle, and the fitted-intercept arm also departs from the
theorem's training assumptions. Since exact-risk selection includes the
calibrated candidate, its gain is nonnegative in every run. Validation selection
has no corresponding guarantee for population risk.

Even within the finite family $\mathcal A$, a theoretical adaptation gain
compares the best fixed candidate with the best realized candidate:
\[
  A_s = \min_{a \in \mathcal A}\E_D[r_R(\widehat f_{D,s,a})]
        - \E_D[\min_{a \in \mathcal A}r_R(\widehat f_{D,s,a})].
\]
Conditional on the frozen calibration choice $a_{s,\mathrm{cal}}$, exact
selection instead gives
\[
  \E_D[G_s]
  = A_s
    + \left\{ \E_D[r_R(\widehat f_{D,s,a_{s,\mathrm{cal}}})] - \min_{a \in \mathcal A}\E_D[r_R(\widehat f_{D,s,a})] \right\}.
\]
The second term measures calibration suboptimality. It vanishes for the
baseline, which chooses the finite-family optimum using exact expected moments.
For the fitted-intercept example, it is $0.01200$ for U and zero for SD
(\cref{sec:bridge-exact-calibration}). The observed gains therefore combine
adaptation with correction of U's suboptimal calibration choice; their Monte
Carlo standard errors describe evaluation uncertainty conditional on those
choices. The theorem's optimization over deterministic rates and stops does not
cover $A_s$ or these empirical gains.

\begin{table}[htbp]
  \centering
  \begin{tabular}{llrrrr}
\toprule
Selection & Scaling & Median & 5th pct. & 95th pct. & Worse (\%) \\
\midrule
\multicolumn{6}{l}{\emph{Fitted intercept}} \\
Exact risk & U & $0.1456$ & $0.0067$ & $0.7747$ & $0.0$ \\
Exact risk & SD & $0.0626$ & $0.0012$ & $0.4083$ & $0.0$ \\
Validation (2048) & U & $0.1204$ & $-0.0268$ & $0.7455$ & $13.3$ \\
Validation (2048) & SD & $0.0477$ & $-0.0282$ & $0.3915$ & $22.3$ \\
\midrule
\multicolumn{6}{l}{\emph{Fixed zero intercept (baseline)}} \\
Exact risk & U & $0.0913$ & $0.0049$ & $0.4214$ & $0.0$ \\
Exact risk & SD & $0.0679$ & $0.0015$ & $0.4214$ & $0.0$ \\
Validation (2048) & U & $0.0688$ & $-0.0352$ & $0.4021$ & $17.9$ \\
Validation (2048) & SD & $0.0521$ & $-0.0290$ & $0.4106$ & $20.8$ \\
\bottomrule
\end{tabular}

  \caption{
    Selection-gain distributions for the example in \cref{tab:bridge-selection}:
    fresh observations, $q = 0.02$, $M = 204$, and
    $(\beta_C^*,\beta_R^*) = (0.1,1)$. Quantiles use 4096 evaluation runs.
    ``Worse'' is the percentage with negative gain, meaning higher population
    rare risk than the calibrated predictor of the same scaling.
  }
  \label{tab:bridge-gain-distributions}
\end{table}

\begin{figure}[htbp]
  \centering
  \includegraphics[]{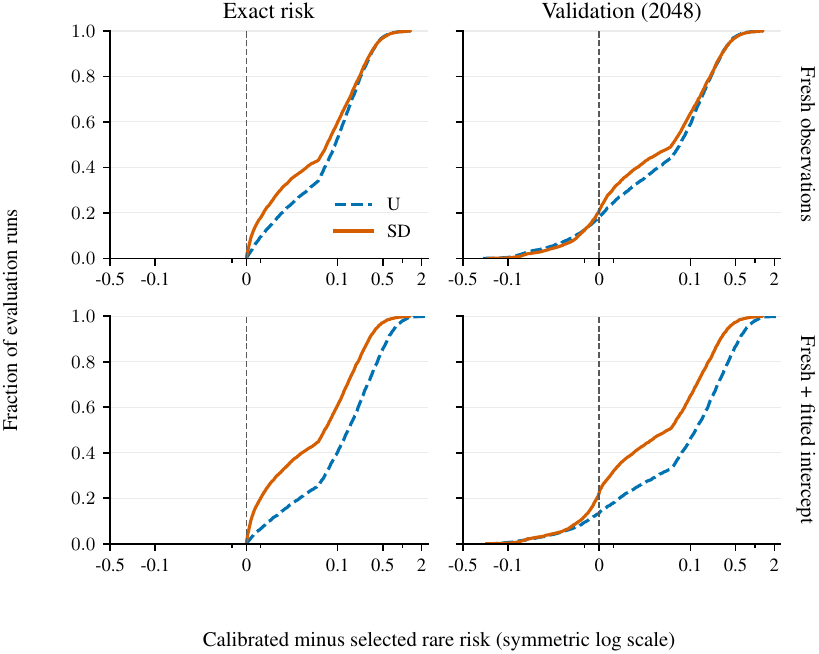}
  \caption{
    Empirical cumulative distributions of selection gains for
    \cref{tab:bridge-gain-distributions}. A curve farther to the right indicates
    larger gains. The horizontal scale is symmetric logarithmic, linear for
    absolute gains at most $0.05$; negative gains mean selection worsens risk.
    Both arms and scalings use the same 4096 evaluation streams.
  }
  \label{fig:bridge-selection-gains}
\end{figure}

The mean gains conceal substantial variation across runs
(\cref{tab:bridge-gain-distributions,fig:bridge-selection-gains}). With a fitted
intercept and exact selection, U's gain has median $0.1456$ and 5th--95th
percentiles $0.0067$--$0.7747$, compared with $0.0626$ and $0.0012$--$0.4083$
for SD. The largest $10\%$ of gains contribute $35.5\%$ of U's total and
$39.3\%$ of SD's total. Thus larger U gains appear in typical runs as well as in
the upper tail. The paired difference $G_U - G_{\mathrm{SD}}$ has mean $0.1210$
(MCSE $0.00378$), median $0.0599$, and 5th--95th percentiles
$-0.1706$--$0.5962$. U gains more in $71.3\%$ of runs, so its advantage in mean
gain does not extend to every run. The per-run identity
\[
  \mathrm{gap}_{\mathrm{sel}}(D) = \mathrm{gap}_{\mathrm{cal}}(D)
                                   + G_U(D)
                                   - G_{\mathrm{SD}}(D)
\]
explains why the mean gain difference overturns the calibrated gap $-0.1046$.
For the fixed-intercept baseline, the mean paired gain difference is only
$0.0133$ (MCSE $0.00086$), with median $0.0131$ and 5th--95th percentiles
$-0.0699$--$0.0915$. U gains more in $68.2\%$ of runs, but this smaller
difference does not overturn the calibrated gap $-0.0192$.

The saved predictors show where the remaining rare risk lies. Let
$e_R = a + (1 - q)(b_R - \beta_R^*)$ denote their signed prediction error
averaged across the two rare cells, and let $e_C = b_C - \beta_C^*$. Under exact
selection with a fitted intercept, the root mean square errors for $e_R$ fall
from $0.6735$ to $0.0255$ for U and from $0.5194$ to $0.1843$ for SD\@. The
corresponding values for $e_C$ fall from $0.3365$ to $0.0775$ for U and from
$0.1082$ to $0.0647$ for SD\@. Since $r_R = e_C^2/8 + e_R^2/2$, the remaining
rare-group mean error dominates SD's selected risk, while U nearly removes it.
In the baseline, selected U and SD predictors instead have root mean square
errors for $e_R$ of $0.1528$ and $0.1269$, and for $e_C$ of $0.1564$ and
$0.0796$. SD retains smaller errors of both kinds and hence lower mean risk.
These comparisons locate the remaining error; the decomposition below separates
the intercept and rare-coefficient contributions.

With 2048 validation rows, fitted-intercept U gains more than SD in $68.0\%$ of
runs, but selection worsens risk relative to calibration in $13.3\%$ of U runs
and $22.3\%$ of SD runs. The baseline also has negative gains, in $17.9\%$ and
$20.8\%$ of runs, respectively. Thus a positive mean improvement does not
guarantee an improvement on each training stream, even when every validation set
contains rare observations. The source supplement retains distributions and
per-run risks, gains, and predictor errors for both fresh arms, both
prevalences, all three caps, both signals, and exact-risk selection and all
three validation sizes. Independent predictions on the two rare cells verify
these results. The reproduction guide explains how to reconstruct the analysis
from saved predictors without training.

\subsubsection{Intercept--rare error cancellation}
\label{sec:bridge-cancellation}

Write each saved predictor in the population-centered coordinates used above.
The true centered intercept is zero, so its error is $a$, not the raw intercept
saved with the coefficients. Let $d = (1 - q)(b_R - \beta_R^*)$. Independence
and the common feature's prevalence of $1/2$ within the rare subgroup give the
per-run identity
\[
  r_R = \underbrace{e_C^2/8}_{\text{common}}
        + \underbrace{a^2/2}_{\text{intercept}}
        + \underbrace{d^2/2}_{\text{rare}}
        + \underbrace{ad}_{\text{cross}}.
\]
The cross term is negative when the intercept and rare-coefficient errors oppose
each other. Define their pooled cancellation fraction as
$-\E[ad]/\E[(a^2+d^2)/2]$, using evaluation means. It equals one for complete
cancellation and is negative when the errors reinforce each other. The two
squared-error terms and their cross term must be considered together: accurate
rare predictions do not require either coordinate to be estimated accurately.

\begin{table}[htbp]
  \centering
  \begin{tabular}{llrrrrr}
\toprule
Selection & Scaling & Common & Intercept & Rare & Cross & Total \\
\midrule
\multicolumn{7}{l}{\emph{Fitted intercept}} \\
Calibrated & U & $0.0142$ & $0.0573$ & $0.1677$ & $0.0018$ & $0.2409$ \\
Calibrated & SD & $0.0015$ & $0.0013$ & $0.1326$ & $0.0009$ & $0.1364$ \\
Exact risk & U & $0.0008$ & $0.2208$ & $0.2235$ & $-0.4439$ & $0.0011$ \\
Exact risk & SD & $0.0005$ & $0.0033$ & $0.0201$ & $-0.0064$ & $0.0175$ \\
Validation (2048) & U & $0.0131$ & $0.2311$ & $0.2362$ & $-0.4536$ & $0.0268$ \\
Validation (2048) & SD & $0.0046$ & $0.0041$ & $0.0331$ & $-0.0074$ & $0.0344$ \\
\midrule
\multicolumn{7}{l}{\emph{Fixed zero intercept (baseline)}} \\
Calibrated & U & $0.0174$ & $0.0000$ & $0.1354$ & $0.0000$ & $0.1528$ \\
Calibrated & SD & $0.0015$ & $0.0000$ & $0.1322$ & $0.0000$ & $0.1337$ \\
Exact risk & U & $0.0031$ & $0.0000$ & $0.0117$ & $0.0000$ & $0.0147$ \\
Exact risk & SD & $0.0008$ & $0.0000$ & $0.0081$ & $0.0000$ & $0.0088$ \\
Validation (2048) & U & $0.0146$ & $0.0000$ & $0.0247$ & $0.0000$ & $0.0393$ \\
Validation (2048) & SD & $0.0046$ & $0.0000$ & $0.0214$ & $0.0000$ & $0.0259$ \\
\bottomrule
\end{tabular}

  \caption{
    Mean rare-risk contributions for the primary example in
    \cref{tab:bridge-selection}, using the same 4096 paired evaluation runs.
    Columns give $e_C^2/8$, $a^2/2$, $d^2/2$, $ad$, and their sum. Rounding can
    obscure exact addition. The baseline's centered intercept and cross terms
    are exactly zero.
  }
  \label{tab:bridge-risk-decomposition}
\end{table}

For fitted-intercept U under exact selection, the squared intercept and
rare-error terms have means $0.2208$ and $0.2235$, while the cross term is
$-0.4439$ (\cref{tab:bridge-risk-decomposition}). Their sum is only $0.000325$.
The pooled cancellation fraction is $99.93\%$, compared with $27.37\%$ for SD\@.
U's errors oppose each other in $99.2\%$ of runs. Its intercept and rare-error
root mean squares actually increase from $0.3384$ and $0.5792$ under calibration
to $0.6645$ and $0.6685$ under exact selection. Calibration has little
cancellation for either scaling. Thus U's lower selected risk reflects closely
matched, opposing errors rather than more accurate intercept and rare
coefficients.

Separating the paired SD-minus-U gap gives a direct accounting test. Under exact
selection, the common, intercept, and rare terms sum to a gap of $-0.42112$
(MCSE $0.00667$); the cross-term gap adds $+0.43754$ (MCSE $0.00687$), producing
the observed $+0.01643$ (MCSE $0.00128$). Relative to calibration, the cross
term changes the gap by $+0.43838$, while the other three contributions together
change it by $-0.31740$. Cancellation therefore accounts for the larger U gain
that overturns SD's calibrated advantage. Removing the cross term is an
accounting diagnostic, not a risk attainable by simply deleting a model
coordinate. If we instead set the centered intercept to zero in these already
selected predictors, their mean gap becomes $-0.20360$. This intervention does
not retrain or reselect them and does not estimate the fixed-intercept arm. In
that separately trained baseline, both intercept contributions vanish and SD
retains smaller common and rare errors under exact selection.

The same decomposition supports the primary validation comparison: with 2048
validation rows, the pooled cancellation fractions are $97.1\%$ for U and
$19.9\%$ for SD\@. The squared-error terms give a gap $-0.43865$, and the cross
term adds $+0.44619$, yielding $+0.00754$. Across both prevalences and all three
caps, exact selection reverses all six fitted-intercept comparisons with the
dominant rare signal. In each, omitting the cross term restores the calibrated
sign. Neither comparisons with a dominant common signal nor any of the 12
baseline comparisons reverse under exact selection. Validation also produces
reversals in the fixed-intercept baseline, where this cancellation is
impossible, so the mechanism does not explain every selection reversal.

These are descriptive identities for the selected predictors in a finite
candidate family, conditional on calibration. They identify how cancellation
accounts for this reversal, without identifying a causal effect of altering
intercept training. \Cref{sec:bridge-exact-calibration} separately quantifies
calibration suboptimality. The source supplement retains per-run errors, four
risk terms, paired gap and gain contributions, and cancellation summaries for
all 120 combinations of settings and selection policies in the two fresh arms.
Independent four-cell predictions and the existing risks and gains verify the
decomposition. The reproduction guide explains how to reconstruct the analysis
without training.

\subsection{Numerical checks}

Independent implementations check transformed-coordinate SGD against
original-coordinate updates, the analytical losses against four-cell
enumeration, and baseline moments against a full tensor recurrence. Replaying
one evaluation replication from every shard checks 73,728 selected endpoints,
with maximum coefficient discrepancy below $1.7 \times 10^{-14}$. All saved
selected predictors pass the cell-enumeration risk check. No calibration or
evaluation candidate fails numerically, and no calibrated choice on the fine
grid lies at a boundary of the positive rate range.

For the fresh fixed-intercept baseline, a denser exact search uses 1025 positive
rate factors and every integer stop. The largest excess of the chosen fine-grid
risk over this denser minimum is $0.000109$ at $q = 0.02$ and $0.000204$ at
$q = 0.005$. Evaluation means differ from their exact expectations by at most
$1.98$ Monte Carlo standard errors. These checks assess the baseline grids and
simulation, not continuous optimality for all arms.

The source supplement contains the protocol, implementation, settings, raw
feature cells and noise, replacement indices, calibration summaries, selected
predictors and their indices, numerical failure records, and numerical checks.
Its machine-readable summary includes aggregate and rare risks, matched changes
between arms, coarse-grid comparisons, selection regret, and coverage
diagnostics. The reproduction guide explains how to summarize and plot the saved
results without training.

\section{Source supplement}
\label{sec:reproducibility}

The source supplement is available at
\url{https://anonymous.4open.science/r/normreg-sgd-1382/}. It contains the paper
source, experiment and plotting code, saved feature draws, results, numerical
checks, and pinned environment definitions. Its \texttt{README.md} gives the
setup, container, and reproduction instructions, including how to build the
complete paper from saved results, rerun experiments, and check the numerical
certificates.

The supplement includes checks of the moment limits, uniform bounds, variational
solver, and saved finite-prevalence selections. These numerical checks
supplement the analytic proofs.

The certificate checker for \cref{thm:conference-predictive} recomputes the
certificates without training models. It checks the six continuous parameter
regions in \cref{sec:predictive-proofs}, including exact rational bounds and
coverage volumes. The supplement retains the recorded output as a summary of the
completed checks, not a substitute for rerunning the checker. The reproduction
guide gives the commands for these checks.

We ran the experiments on a local NixOS workstation with an AMD Ryzen 9 7900 CPU
and 64\,GiB of RAM\@. The Python implementations use NumPy and SciPy on the CPU,
with OpenBLAS restricted to one thread by the experiment tasks. Saved experiment
metadata include software versions and elapsed runtimes.

\end{document}